\documentclass[13pt]{article}
\usepackage{latexsym}
\usepackage{geometry}
\usepackage{graphicx}
\usepackage{amsmath, amssymb, amsthm}
\usepackage{booktabs}
\usepackage{algorithm}
\usepackage{algorithmic}

\usepackage[utf8]{inputenc} % allow utf-8 input
\usepackage[T1]{fontenc}    % use 8-bit T1 fonts

\usepackage{url}
\usepackage{natbib}
\usepackage{appendix}

\usepackage{amsfonts}
\usepackage{multirow}
\usepackage{multicol}

\usepackage{color}

\usepackage{xcolor,colortbl}
\definecolor{LightCyan}{rgb}{0.88,1,1}

\usepackage{subfigure}
\usepackage{hyperref}
\usepackage{tcolorbox}

\usepackage{color}
\usepackage{xcolor,colortbl}

\def\K{\mathcal{K}}
\def\I{\mathbb{I}}
\def\R{\mathbb{R}}
\def\E{\mathbb{E}}

\newcommand{\sign}{\mbox{sign}}

\newtheorem{theorem}{Theorem}
\newtheorem{lemma}{Lemma}

\newtheorem{definition}{Definition}
\newtheorem{assumption}{Assumption}
\newtheorem{remark}{Remark}

\begin{document}
	
\title{ CLion: Efficient Cautious Lion Optimizer with Enhanced Generalization }
\author{Feihu Huang\thanks{Feihu Huang is with College of Computer Science and Technology,
Nanjing University of Aeronautics and Astronautics, Nanjing, China;
and also with MIIT Key Laboratory of Pattern Analysis and Machine Intelligence, Nanjing, China. Email: huangfeihu2018@gmail.com}, \
Guanyi Zhang\thanks{Guanyi Zhang is with College of Computer Science and Technology,
Nanjing University of Aeronautics and Astronautics, Nanjing, China.}, \
Songcan Chen\thanks{Songcan Chen is with College of Computer Science and Technology,
Nanjing University of Aeronautics and Astronautics, Nanjing, China;
and also with MIIT Key Laboratory of Pattern Analysis and Machine Intelligence, Nanjing, China.}
 }

\date{}
\maketitle

\begin{abstract}
 Lion optimizer is a popular learning-based optimization algorithm in machine learning, which
 shows impressive performance in training many deep learning models. Although convergence property of the Lion optimizer has been studied, its generalization analysis is still missing.
 To fill this gap, we study generalization property of the Lion via algorithmic stability based on the mathematical induction. Specifically, we prove that the Lion has a generalization error of $O(\frac{1}{N\tau^T})$, where $N$ is training sample size, and
 $\tau>0$ denotes the smallest absolute value of non-zero element in gradient estimator, and $T$ is the total iteration number. In addition, we obtain an interesting byproduct that the SignSGD algorithm has the same generalization error as the Lion. 
 To enhance generalization of the Lion, we design a novel efficient Cautious Lion (i.e., CLion) optimizer by cautiously using sign function. Moreover, we prove that our CLion has a lower generalization error of $O(\frac{1}{N})$ than $O(\frac{1}{N\tau^T})$ of the Lion, since the parameter $\tau$ generally is very small.
 Meanwhile, we study convergence property of our CLion optimizer, and prove that our CLion has a fast convergence rate of $O(\frac{\sqrt{d}}{T^{1/4}})$ under $\ell_1$-norm of gradient for nonconvex stochastic  optimization, where $d$ denotes the model dimension.
 Extensive numerical experiments demonstrate effectiveness of our CLion optimizer.
\end{abstract}

\section{Introduction}
Efficient optimization algorithms~\citep{bottou2018optimization} have become one of central attention in machine learning, with ever-increasing costs of learning large models. In recent years, training large models has been dominated by the handcrafted optimizers such as stochastic gradient descent (SGD)~\citep{robbins1951stochastic}, momentum-based SGD (SGDM)~\citep{sutskever2013importance},
Adam~\citep{kingma2014adam} and AdamW~\citep{loshchilov2017decoupled} algorithms. For example, Adam~\citep{kingma2014adam} algorithm is designed by using the first-order momentum to obtain gradient estimator and using the second-order momentum to get adaptive learning rate. AdamW~\citep{loshchilov2017decoupled} is designed by using the decoupled weight decay
in Adam algorithm. Recently, these handcrafted optimizers are still the mainstream optimizers for training deep learning models including language models~\citep{vaswani2017attention} and vision models~\citep{wu2020visual}.

Another direction is to automatically find some efficient optimization algorithms via learning algorithms~\citep{harrison2022closer,chen2022scalable,chen2023symbolic}.
More recently, learning-based optimization algorithms have been begun to attract attention in machine learning.
For example, the Lion optimizer~\citep{chen2023symbolic} is a typical learning-based optimization algorithm, which is found by Google via program search. Specifically, the Lion uses the incorporation of two distinct interpolations between the previous step momentum and the current step gradient, as well as the integration of decoupled weight decay.
Lion~\citep{chen2023symbolic} shows impressive
performance in training many deep learning models, and performs comparably or favorably to AdamW but with greater memory efficiency. Recently, some works~\citep{chen2023lion,dong2024convergence,sfyraki2025lions,jiang2025convergence,yu2026sign} studied convergence properties of the Lion. For example, \cite{chen2023lion}
studied convergence properties of
the Lion optimizer for solving a class of bound-constrained optimization problems. \cite{dong2024convergence} further studied convergence properties of the Lion optimizer for solving the nonconvex unconstrained optimization problems.
\cite{sfyraki2025lions} studied convergence properties of the Lion and its variance reduced variant via stochastic Frank-Wolfe~\citep{frank1956algorithm,reddi2016stochastic} under the light-tailed and
heavy-tailed noise settings, respectively.
\cite{jiang2025convergence} proposed a variance reduced Lion and its distributed version, and provided its convergence analysis. Subsequently, \cite{yu2026sign} further studied convergence properties of the Lion
under the heavy-tailed noise setting.

\begin{table*}
	\centering
	\caption{ \textbf{Generalization error} comparison of our CLion optimizer and other representative optimizers. Here $N$ denotes the training sample size, and $\tau$ denotes  the absolute value of the smallest non-zero element in gradient estimator, \emph{which is generally very small}. $T$ is the iteration number.}
	\label{tab:1}
	\resizebox{0.96\textwidth}{!}{
		\begin{tabular}{c|c|c|c}
			\hline
			\textbf{Algorithm} & \textbf{Reference} & \textbf{Generalization Error} & \textbf{Using Sign Function}
			\\ \hline
			SGD  & \cite{hardt2016train} & $O(\frac{1}{N})$  &  \\   \hline
			SGDM  & \cite{ramezani2024generalization} & $O(\frac{1}{N})$  &  \\  \hline
			%Adam  & \cite{zhou2024towards} & $O(\frac{1}{\sqrt{N}})$  &  \\   \hline
			%  AdamW  & \cite{zhou2024towards} & $O(\frac{1}{\sqrt{N}})$ &   \\  \hline
			%  HomeAdam & \cite{huang2026homeadam} & $O(\frac{1}{N})$ & \\  \hline
			SignSGD  & \cite{bernstein2018signsgd} & $O(\frac{1}{N\tau^{T}})$  & $\checkmark$  \\  \hline
			Lion  & \cite{chen2023symbolic} & $O(\frac{1}{N\tau^{T}})$  & $\checkmark$  \\  \hline
			CLion & Ours & \textcolor{red}{$O(\frac{1}{N})$}& \textcolor{red}{$\checkmark$} \\  \hline
		\end{tabular}
	}
\end{table*}

Meanwhile, some variants~\citep{liu2024communication,rong2025refined,yuan2024mars,sfyraki2025lions,jiang2025convergence} of Lion optimizer have been developed.
For example, \cite{liu2024communication} proposed a communication efficient distributed Lion algorithm for distributed optimization and provided its convergence analysis. \cite{yuan2024mars,sfyraki2025lions,jiang2025convergence} independently proposed  a class of variance reduced Lion optimizers based on the momentum-based variance reduced technique~\citep{cutkosky2019momentum,tran2022hybrid}. More recently, \cite{rong2025refined} proposed a refined lion optimizer by using a continuous function arctan instead of sign function and only studied its non-asymptotic convergence
under the convex setting.

Although the convergence property (i.e., optimization error~\citep{bottou2018optimization}) of the Lion optimizer and its variants, its generalization property (i.e., generalization error~\citep{shalev2010learnability}) is still missing.
In fact, optimization error~\citep{bottou2018optimization} only shows
how quickly the empirical training loss decreases. While the ultimate goal of machine learning models has a good performance on unseen examples. Naturally, generalization error~\citep{shalev2010learnability,hardt2016train,zhang2023mathematical} measures
the gap between training loss and population loss, and a small
generalization error shows strong performance on unseen examples.
However, it generally can remain large even
when optimization error is small, leading to overfitting of models.
So far, although optimization error of the Lion optimizer has been studied, its rigorous generalization analysis is still unexplored. To fill this critical blind spot,
we study generalization error of the Lion optimizer, and propose a novel efficient cautious Lion (i.e., CLion) optimizer to improve its generalization.
\subsection*{Contributions}
In the paper, our main contributions are given as follows:
\begin{itemize}
	\item[1)] We study generalization error of the Lion optimizer via algorithmic stability based on the mathematical induction, and prove that it has a generalization error of $O(\frac{1}{N\tau^T})$ under the unconvex setting, where $N$ is the training sample size, and $\tau$ denotes the smallest absolute value of non-zero element in gradient estimator, and $T$ is the total iteration number. Meanwhile, we obtain an interesting byproduct that the SignSGD algorithm~\citep{bernstein2018signsgd} also has a generalization error of $O(\frac{1}{N\tau^T})$, since the Lion algorithm reduces to SignSGD algorithm when $\beta_1=\beta_2=\lambda=0$ in Algorithm~\ref{alg:1}.
	\item[2)] To improve generalization of the Lion, we propose a novel CLion optimizer by cautiously using sign function. Moreover, we prove that our CLion has a lower generalization error of $O(\frac{1}{N})$ than $O(\frac{1}{N\tau^T})$ of the Lion, since the value $\tau$ generally is very small.
	In particular, our CLion reaches the same generalization error of SGD~\citep{hardt2016train} and  momentum-based SGD (i.e., SGDM)~\citep{ramezani2024generalization} (Please see Table~\ref{tab:1}). 
	\item[3)] Meanwhile, we study the convergence property of our CLion optimizer, and prove that our CLion has a fast convergence rate of $O(\frac{\sqrt{d}}{T^{1/4}})$ under $\ell_1$-norm of gradient for nonconvex stochastic optimization, which is the same convergence rate of the Lion~\citep{dong2024convergence,jiang2025convergence}.
	\item[4)] We conduct some numerical experiments on training vision and language models to demonstrate efficiency of the CLion optimizer.
\end{itemize}
%From Table~\ref{tab:1},  not only has a smaller generalization error than the Adam and AdamW algorithms~\citep{zhou2024towards}, but also has a more efficient memory than Adam and AdamW as it only keeps track of the first-order momentum. Meanwhile, our CLion optimizer has the same generalization of the new
%HomeAdam algorithm~\citep{huang2026homeadam}, but still has a more efficient memory than HomeAdam as it only keeps track of the first-order momentum.
\section*{Notations}
$\sign(\cdot)$ denotes a sign function, i.e.,
$\sign(a)=1$ when $a>0$, and $\sign(a)=0$ when $a=0$, otherwise $\sign(a)=-1$.
$\I(\cdot)$ denotes an index function, i.e., $\I(a)=1$ when $a\neq0$, otherwise $\I(a)=0$ for $a=0$.
Let $[N]=\{1,2,\cdots,N\}$. $\R^+$ denotes non-negative real number set.
For vectors $x,y\in \R^{d}$, $\langle x,y\rangle$ denotes inner product.
$\|x\|$ and $\|x\|_1$ denote the $\ell_2$ and $\ell_1$ norms of vector $x$, respectively.
$a_t=O(b_t)$ denotes that $a_t \leq c b_t$ for some constant $c>0$.

\section{Preliminaries}
\subsection{Problem}
In the paper, we study the following
nonconvex stochastic optimization problem
\begin{align}\label{eq:p1}
	\min_{w \in \R^{d}} F(w) = \E_{\xi \sim \mathcal{D}}[f(w;\xi)],
\end{align}
where $f(w;\xi): \R^{d}\rightarrow \R^+$ denotes a loss function on a sample $\xi\sim \mathcal{D}$,
which is possibly nonconvex. Here $\xi$ is a random variable drawn
some fixed but unknown distribution $\mathcal{D}$. $F(w) = \E_{\xi\sim \mathcal{D}}[f(w;\xi)]$
denotes a population loss (risk) of machine learning tasks such as training deep learning models.
In general, we only access a finite set of
training data $S=\{\xi_1,\xi_2,\cdots,\xi_N\}$ drawn i.i.d. from $\mathcal{D}$, since the fixed distribution $\mathcal{D}$ is unknown. Thus we could use the empirical risk
\begin{align}
	F_S (w) = \frac{1}{N}\sum_{i=1}^N f(w;\xi_i)
\end{align}
to approximate the population risk $F(w)= \E_{\xi \sim \mathcal{D}}[f(w;\xi)]$.
\subsection{ Definition of Generalization Error }
Given a training dataset $S$, we run a
randomized algorithm $A$ to minimize the empirical risk to get a model $A(S)$.
In fact, it does not necessarily show that the output model $A(S)$ would have a
good performance on test examples, which is measured
by the population risk $F(w) =\E_{\xi \sim \mathcal{D}}[f(w;\xi)]$.
We are interested in the excess
population risk $F(A(S))-F(w^*)$, which measures the relative behavior of the output model as compared
to the best model
\begin{align}
	w^*=\mathop{\arg\min}_{w\in \R^{d}}F(w). \nonumber
\end{align}
Then we can decompose this excess population risk into the following formation
\begin{align}
	F(A(S))-F(w^*)= \underbrace{F(A(S)) - F_S(A(S))}_{(i)}  + \underbrace{F_S(A(S)) - F_S(w^*)}_{(ii)} + F_S(w^*) -F(w^*), \label{eq:ep}
\end{align}
where the term $(i)$ $F(A(S)) - F_S(A(S))$ is generalization error (generalization gap), which measures
the gap between training loss and population loss,
and the term $(ii)$ $F_S(A(S)) - F_S(w^*)$ is optimization error, which quantifies how well the algorithm minimizes the empirical risk. Taking expectation on this inequality~(\ref{eq:ep}) with random algorithm $A$ and
training dataset $S$, since $\E_{A,S}[F_S(w^*) -F(w^*)]=0$, i.e., $w^*$
is independent of $A$ and $S$, we have
\begin{align}
	\E_{A,S} [F(A(S)) -F(w^*) ]&= \E_{A,S} [ F(A(S)) - F_S(A(S))]  + \E_{A,S} [F_S(A(S)) - F_S(w^*)]. \nonumber
\end{align}
\subsection{ Definition of Algorithmic Stability }
In the paper, we study the generalization by using the algorithmic stability.
We first introduce the uniform stability, on-average stability
and generalization gap.
For notational simplicity, let $S=\{\xi_1,\xi_2,\cdots,\xi_N\}$ and $\tilde{S}=\{\tilde{\xi}_1,\tilde{\xi}_2,\cdots,\tilde{\xi}_N\}$ be two
independent datasets drawn from the same distribution $\mathcal{D}$.  Here we denote
the dataset $S^{(i)}=\{\xi_1,\xi_2,\cdots,\tilde{\xi}_i,\cdots,\xi_N\}$ by replacing the $i$-th example $\xi_i$ with an independent sample $\tilde{\xi}_i$ for any $i\in [N]$.

\begin{definition} (\textbf{Stability in Loss Function})
	Let $A$ be a random algorithm and $A(S)$ denote the output of the algorithm $A$ run on dataset $S$. If for any $S$ and $S^{(i)}$
	$$\sup_{\xi} \E_A [f(A(S),\xi)-f(A(S^{(i)}),\xi)] \leq \epsilon,$$
	the random algorithm $A$ is $\epsilon$-uniform stable.
\end{definition}
\begin{lemma} \label{lem:gs}
	(\textbf{Generalization via Stability})~\citep{shalev2010learnability,hardt2016train}. Let algorithm $A$ be $\epsilon$-uniformly stable in function values. Then we have
	$$ |\E_{A,S}[F(A(S)) - F_S(A(S))]|\leq \epsilon.$$
\end{lemma}

\begin{figure}[ht]
	\centering
	\subfigure[]{\includegraphics[width=0.21\textwidth]{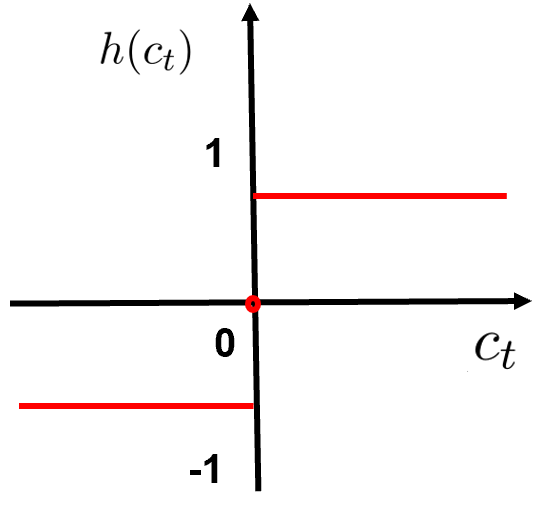}}
	\hfill
	\subfigure[]{\includegraphics[width=0.25\textwidth]{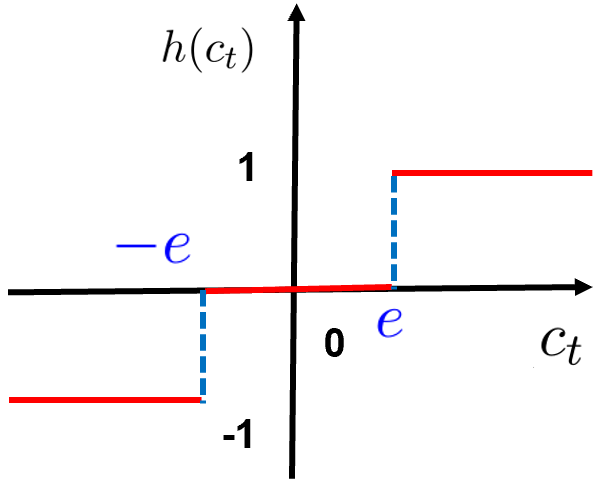}}
	\hfill
	\subfigure[]{\includegraphics[width=0.23\textwidth]{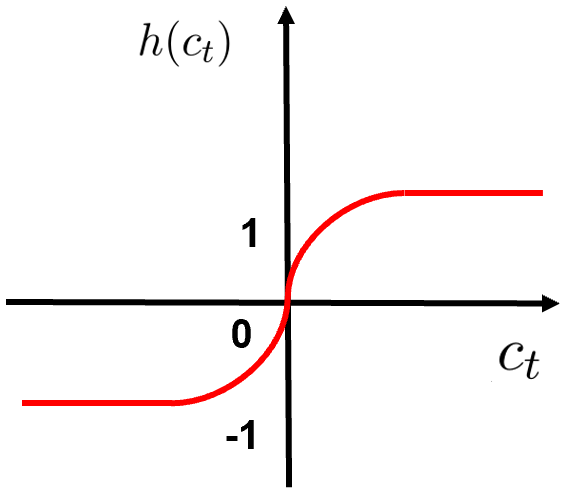}}
	\hfill
	\subfigure[]{\includegraphics[width=0.25\textwidth]{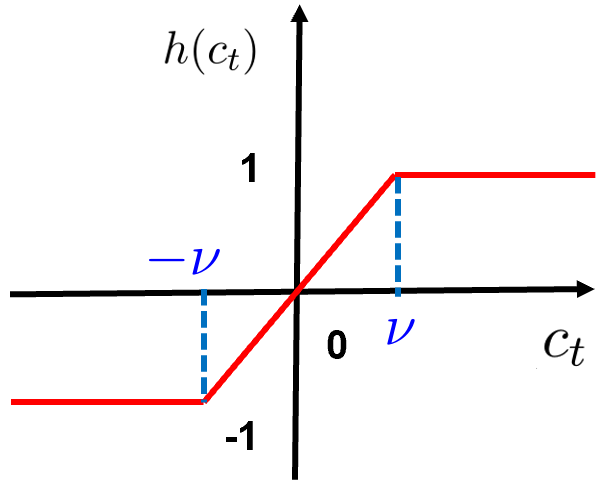}}
	\hfill
	\caption{ Illustration of different
		active functions: \textbf{(a)} Lion~\citep{chen2023symbolic} uses $h(c_t)=\sign(c_t)$; \textbf{(b)} Lion-$\K$~\citep{chen2023lion} uses $h(c_t)=\nabla\K(c_t)=\I(|c_t|>e)\sign(c_t)$, where $e>0$; \textbf{(c)} Lion-$\K$~\citep{chen2023lion} uses $h(c_t)=\nabla\K(c_t)=\mbox{tanh}(ac_t)$, where $a>0$, and RLion~\citep{rong2025refined} uses $h(c_t)=\frac{2}{\pi}\mbox{arctan}(\alpha c_t)$, where $\alpha>0$;
		\textbf{(d)} our CLion uses $h(c_t)=\sign(c_t)$ when $\min_{j\in S_t} |(c_t)_j| \geq \nu>0$, otherwise $h(c_t)=c_t$, where $S_t=\{j|(c_t)_j\neq 0, j=1,\cdots,d\}$.}
	\label{fig:1}
\end{figure}
\section{Cautious Lion Optimizer}
In this section, we propose an efficient cautious Lion (i.e., CLion) optimizer to solve
the above nonconvex stochastic problem~(\ref{eq:p1}) by cautiously using sign function based on the 
Lion optimizer~\citep{chen2023symbolic}. Algorithm~\ref{alg:1} shows an an algorithmic framework for 
the Lion optimizer~\citep{chen2023symbolic}. When $\beta_1=\beta_2=\lambda=0$ in Algorithm~\ref{alg:1}, 
it also reduces to the SignSGD algorithm~\citep{bernstein2018signsgd}.

\begin{algorithm}
	\caption{\textbf{Lion} Optimizer~\citep{chen2023symbolic} }
	\label{alg:1}
	\begin{algorithmic}[1]
		\STATE \textbf{Input}: $\eta>0$, $\beta_1,\beta_2\in [0,1)$ and $\lambda\geq 0$;
		\STATE \textbf{Initialize:} $w_0\in \R^{d}$
		and $m_0=0$;
		\FOR{$t = 1, 2, \ldots, T$}
		\STATE  Draw a sample $\xi_{t} \sim \mathcal{D}$;
		\STATE $g_t = \nabla f(w_{t-1};\xi_{t})$;
		\STATE $c_{t}= \beta_{1}m_{t-1} + (1-\beta_{1})g_t$;
		\STATE $w_{t} = w_{t-1} - \eta (\sign(c_t)+\lambda w_{t-1})$;
		\STATE $m_{t} = \beta_{2}m_{t-1} + (1-\beta_{2})g_t$.
		\ENDFOR
		\STATE \textbf{Output:} $w_{T}$.
	\end{algorithmic}
\end{algorithm}

Algorithm~\ref{alg:2} provides an algorithmic framework of our CLion optimizer.
At the lines 6 and 12 of Algorithm~\ref{alg:2}, our CLion optimizer uses the same gradient estimator $c_t$
as the standard Lion optimizer~\citep{chen2023symbolic} shown in Algorithm~\ref{alg:1}.
In Algorithm~\ref{alg:2}, when $\min_{ j\in S_t} |(c_t)_j| \geq \nu$ with $S_t=\{j|(c_t)_j\neq 0, j=1,\cdots,d\}$, i.e., the smallest absolute value of non-zero element in gradient estimator $c_t$
is larger than a threshold $\nu>0$,
we update the variable $w$ as the standard Lion algorithm, defined as 
\begin{align}
	w_{t} = w_{t-1} - \eta (\sign(c_t)+\lambda w_{t-1}).
\end{align}
Otherwise, we update the variable $w$ as follows:
\begin{align}
	w_{t} = w_{t-1} - \eta (c_t+\lambda w_{t-1}).
\end{align}
Here we define a unified framework of active function $h(\cdot)$ to update variable $w$ as follows:
\begin{align}
	w_{t} = w_{t-1} - \eta (\textcolor{red}{h(c_t)}+\lambda w_{t-1}).
\end{align}

Figure~\ref{fig:1} (a) shows that the Lion~\citep{chen2023symbolic} optimizer uses $h(c_t)=\sign(c_t)$.  Figure~\ref{fig:1} (b) shows that the Lion-$\K$~\citep{chen2023lion} uses $h(c_t)=\nabla\K(c_t)=\I(|c_t|>e)\sign(c_t)$, where $e>0$. Figure~\ref{fig:1} (c) shows that Lion-$\K$~\citep{chen2023lion} uses $h(c_t)=\nabla\K(c_t)=\mbox{tanh}(ac_t)$, where $a>0$, and it also could show that the RLion~\citep{rong2025refined} uses $h(c_t)=\frac{2}{\pi}\mbox{arctan}(\alpha c_t)$, where the curve parameter $\alpha>0$.
In Figure~\ref{fig:1} (d), our CLion uses $h(c_t)=\sign(c_t)$ when $\min_{j\in S_t} |(c_t)_j| \geq \nu>0$, otherwise $h(c_t)=c_t$, where $S_t=\{j|(c_t)_j\neq 0, j=1,\cdots,d\}$.
We could find that our active function not only has a good continuous property,
but also more approaches the discontinuous sign function than other active functions.

\begin{algorithm}
	\caption{ Cautious Lion (\textbf{CLion}) Optimizer}
	\label{alg:2}
	\begin{algorithmic}[1]
		\STATE \textbf{Input}: $\eta>0$, $\beta_1,\beta_2\in (0,1)$, $\lambda> 0$ and $\nu>0$;
		\STATE \textbf{Initialize:} $w_0\in \R^{d}$
		and $m_0=0$;
		\FOR{$t = 1, 2, \ldots, T$}
		\STATE  Draw a sample $\xi_{t} \sim \mathcal{D}$;
		\STATE $g_t = \nabla f(w_{t-1};\xi_{t})$;
		\STATE $c_{t}= \beta_{1}m_{t-1} + (1-\beta_{1})g_t$;
		\IF {$\min_{ j\in S_t} |(c_t)_j| \geq \nu$ with $S_t=\{j|(c_t)_j\neq 0, j=1,\cdots,d\}$}
		\STATE $w_{t} = w_{t-1} - \eta (\sign(c_t)+\lambda w_{t-1})$;
		\ELSE
		\STATE $w_{t} = w_{t-1} - \eta (c_t+\lambda w_{t-1})$;
		\ENDIF
		\STATE $m_{t} = \beta_{2}m_{t-1} + (1-\beta_{2})g_t$.
		\ENDFOR
		\STATE \textbf{Output:} $w_{T}$.
	\end{algorithmic}
\end{algorithm}

Compared to the Lion~\citep{chen2023symbolic}, our CLion optimizer has the following advantages: 1) Our CLion cautiously uses the identity function instead of
the sign function for some small absolute values of non-zero element in gradient estimator $c_t$ below the threshold $\nu>0$. Thus our CLion could relieve the gradient explosion case of Lion optimizer. 2) From the following generalization analysis, our CLion has a lower
generalization error than the Lion.

Compared to the RLion~\citep{rong2025refined}, our CLion has the following advantages: 1) Although our CLion and the RLion use continuous active functions instead of discontinuous sign function used in Lion (see Figure~\ref{fig:1}),
our active function more approaches the sign function,
so our CLion keep more good properties of the Lion. For example,
from the following convergence analysis, our CLion has a fast convergence rate as the Lion under the nonconvex setting.
%While the RLion only provides its non-asymptotic convergence
%under the convex setting.
Meanwhile, our CLion also shows better performances than the RLion in the following numerical experiments.
2) Our CLion could easily choose the threshold $\nu>0$ based on the smallest absolute value of non-zero element in gradient estimator $c_t$. While the curve parameter $\alpha$ in RLion algorithm can not easily control in training process, which is totally rely on manually set.
Meanwhile, by choosing a suitable threshold $\nu>0$,
our CLion also has a lower
generalization error than the Lion (see the following generalization analysis).

Note that although the Lion-$\mathcal{K}$~\citep{chen2023lion} optimizer could use a class of
active functions including $h(c_t)=\nabla\K(c_t)=\I(|c_t|>e)\sign(c_t)$ and $h(c_t)=\nabla\K(c_t)=\mbox{tanh}(ac_t)$ (see Figure~\ref{fig:1} (b) (c)), it only is suitable for
solving a class of bound-constrained optimization problems. While our Clion optimizer does not rely on the bound constraint in solving nonconvex optimization.

\section{Generalization Analysis}
In this section, we provide generalization analysis for the Lion and our CLion optimizers, respectively.
We first give some mild conditions for this generalization analysis.

\begin{assumption}[\textbf{Smoothness of Component Function}]\label{ass:s1}
	Each component function $f(w;\xi)$ is $L$-Lipschitz smooth, such that
	\begin{align}
		\|\nabla f(w_1;\xi)-\nabla f(w_2;\xi)\| \leq L\|w_1-w_2\|, \ w_1, w_2\in \R^{d}.
	\end{align}
\end{assumption}

\begin{assumption}[\textbf{Lipschitzness}]\label{ass:g}
	Each component function $f(w;\xi)$ for all $\xi\sim \mathcal{D}$ is Lipschitz continuous, such that
	\begin{align}
		\|f(w_1;\xi)-f(w_2;\xi)\| \leq G\|w_1-w_2\|, \ w_1, w_2\in \R^{d}, \ G>0.
	\end{align}
\end{assumption}

\begin{assumption}[\textbf{Bounded Variance}]\label{ass:v}
	$\nabla f(w;\xi)$ is an unbiased stochastic estimator of the true gradient $\nabla F(w)$ and has a bounded variance, i.e.,
	\begin{align}
		\E[\nabla f(w;\xi)]=\nabla F(w), \ \E\|\nabla f(w;\xi)-\nabla F(w)\|^2]\leq \sigma^2.
	\end{align}
\end{assumption}

Assumption~\ref{ass:s1} shows smoothness of each component function $f(w;\xi)$, which is widely used in generalization analysis~\citep{hardt2016train,lei2020fine,ramezani2024generalization}.
Assumption~\ref{ass:g} provides the Lipschitz continuous of objective function, which is widely used in
generalization analysis~\citep{hardt2016train,lei2020fine,lei2023stability,ramezani2024generalization}.
Assumption~\ref{ass:v} shows a standard bounded variance assumption used in stochastic optimization~\citep{bottou2018optimization,ghadimi2013stochastic}. According to Assumptions~\ref{ass:s1} and~\ref{ass:v}, we have $\|\nabla F(w_1)-\nabla F(w_2)\| = \|\E[\nabla f(w_1;\xi)-\nabla f(w_2;\xi)]\| \leq \E \|\nabla f(w_1;\xi)-\nabla f(w_2;\xi)\| \leq L\|w_1-w_2\|$. Thus, we could use smoothness of each component function $f(w;\xi)$ to obtain smoothness of population function $F(w)$.

\subsection{Generalization Error of Lion Optimizer}
In this subsection, we provide generalization analysis for the Lion optimizer.
All detailed proofs are provided in Appendix~\ref{ga:lion}.

\begin{lemma} \label{lem:1}
	Assume the sequences $\{c_t\}_{t=1}^T$ and $\{c_t^{(i)}\}_{t=1}^T$ are generated from Algorithm~\ref{alg:1} based on the dataset $S$ and
	$S^{(i)}$, respectively, we have
	\begin{align}
		\big\|\sign(c_t)-\sign(c_t^{(i)})\big\|  \leq \frac{2\sqrt{d}}{\tau}\|c_t-c_t^{(i)}\|,
	\end{align}
	where $\tau=\min_{t\geq 1}(\min_{j\in S_t}(|(c_t)_j|))>0$ with $S_t=\big\{j|\ |(c_t)_j|\neq 0, j=1,2\cdots,d\big\}$.
\end{lemma}

\begin{theorem} \label{th:1}
	Assume the sequence $\{w_t,c_t\}_{t=1}^T$ is generated from Algorithm~\ref{alg:1} on dataset $S=\{\xi_1,\xi_2,\cdots,\xi_N\}$. Under the Assumptions~\ref{ass:s1},~\ref{ass:g},~\ref{ass:v}, let $\eta=O(\frac{1}{\sqrt{d}})$, $\lambda=O(1)$ with $\lambda\in [0, \frac{1}{\eta})$, $\beta_1=O(1)$ with $\beta_1\in [0,1)$, $\beta_2=O(1)$ with $\beta_2\in [0,1)$, $\sigma=O(1)$ and
	$L=O(1)$, we have
	\begin{align}
		|\E[F(w_T) - F_S(w_T)]| \leq O(\frac{1}{\tau^T N}),
	\end{align}
	where $\tau=\min_{t\geq 1}(\min_{j\in S_t}(|(c_t)_j|))>0$ with $S_t=\big\{j|\ |(c_t)_j|\neq 0, j=1,2\cdots,d\big\}$.
\end{theorem}

\begin{proof}
	Here we provide a sketched proof. In this proof, we use the mathematical induction to obtain
	the above result.
	We first prove
	\begin{align}
		\E \|m_1 - m_1^{(i)}\|  \leq \frac{2(1-\beta_2)\sigma}{N}=\frac{\psi_1}{N},
	\end{align}
	where $\psi_1=2(1-\beta_2)\sigma$. We also prove
	\begin{align}
		\E \|w_1 - w_1^{(i)}\| \leq  \frac{2\eta\sqrt{d}}{\tau}\frac{2(1-\beta_1)\sigma}{N} = \frac{\phi_1}{\tau N},
	\end{align}
	where $\phi_1=4\eta\sqrt{d}(1-\beta_1)\sigma$.
	Let $\eta=O(\frac{1}{\sqrt{d}})$, $\beta_1=O(1)$ with $\beta_1\in [0,1)$ and $\sigma=O(1)$, we have
	$\phi_1=4\eta\sqrt{d}(1-\beta_1)\sigma=O(1)$ and $\psi_1=2(1-\beta_2)\sigma=O(1)$.
	Then we have
	\begin{align}
		\E \|w_1 - w_1^{(i)}\|  \leq O(\frac{1}{\tau N}).
	\end{align}
	
	Based on mathematical induction, we assume $\E \|w_t - w_t^{(i)}\| \leq \frac{\phi_t}{\tau N}$ with
	$\phi_t=O(\frac{1}{\tau^{t-1}})$, and $\E\|m_t-m_t^{(i)}\|\leq \frac{\psi_t}{N}$ with $\psi_t=O(\frac{1}{\tau^{t-1}})$. Then we prove
	\begin{align}
		\E\|m_{t+1} - m_{t+1}^{(i)}  \|
		& \leq  \frac{\beta_2\psi_t}{N}  + \frac{2(1-\beta_2)\sigma}{N} + \frac{(1-\beta_2)L\phi_t}{\tau N}  = \frac{\phi_{t+1}}{N},
	\end{align}
	where $\psi_{t+1}= \beta_2\psi_t + 2(1-\beta_2)\sigma + (1-\beta_2)L\frac{\phi_t}{\tau}$.
	Let $\phi_{t+1} = (1-\lambda\eta)\phi_t + 2\eta\sqrt{d}\big( \beta_1 \psi_t + 2(1-\beta_1)\sigma + (1-\beta_1)L\frac{\phi_t}{\tau}\big) $, we have
	\begin{align}
		\E \|w_{t+1} - w_{t+1}^{(i)}\| & \leq (1-\lambda\eta)\E \|w_{t} -w_{t}^{(i)}\| + \frac{2\eta\sqrt{d}}{\tau}\E \|c_{t+1}-c_{t+1}^{(i)}\|  \nonumber \\
		& \leq  (1-\lambda\eta)\frac{\phi_t}{\tau N} + \frac{2\eta\sqrt{d}}{\tau}\big(\frac{\beta_1 \psi_t}{N} +\frac{ 2(1-\beta_1)\sigma}{N} + \frac{(1-\beta_1)L\phi_t}{\tau N} \big)  = \frac{\phi_{t+1}}{\tau N}.
	\end{align}
	Let $\eta=O(\frac{1}{\sqrt{d}})$, $\lambda=O(1)$ with $\lambda\in [0, \frac{1}{\eta})$, $\beta_1=O(1)$ with $\beta_1\in [0,1)$,  $\beta_2=O(1)$ with $\beta_2\in [0,1)$, $\sigma=O(1)$ and
	$L=O(1)$. Since $\psi_t=O(\frac{1}{\tau^{t-1}})$ and $\phi_t=O(\frac{1}{\tau^{t-1}})$, we can obtain $\psi_{t+1}=O(\frac{1}{\tau^{t}})$ and $\phi_{t+1} =O(\frac{1}{\tau^{t}})$.
	Then we have
	\begin{align}
		\E \|w_{t+1} - w_{t+1}^{(i)}\|  \leq  \frac{\phi_{t+1}}{\tau N} =O(\frac{1}{\tau^{t+1} N}).
	\end{align}
	By using mathematical induction, then we have
	\begin{align} \label{eq:18}
		\E \|w_{T} - w_{T}^{(i)}\| \leq O(\frac{1}{\tau^T N}).
	\end{align}
	
	By using Assumption~\ref{ass:g}, i.e., the condition of $G$-Lipschitz $f(w;\xi)$ (i.e.,), we have
	for any $\xi\sim \mathcal{D}$
	\begin{align}  \label{eq:19}
		\E |f(w_T;\xi)-f(w_T^{(i)};\xi)| \leq G \E \|w_{T} - w_{T}^{(i)}\| \leq O(\frac{1}{\tau^T N}),
	\end{align}
	where the last inequality holds by the above inequality~(\ref{eq:18}) and $G=O(1)$.
	
	By using the lemma~\ref{lem:gs}, i.e., the uniform stability bound~\citep{shalev2010learnability,hardt2016train}, and taking expectations over $S$, $S^{(i)}$ and the algorithm's randomness on the above inequality~(\ref{eq:19}), we
	can obtain
	\begin{align}
		|\E [F(w_T) - F_S(w_T)]| \leq O(\frac{1}{\tau^T N}).
	\end{align}
	
\end{proof}
\begin{remark}
	In Theorem~\ref{th:1}, we provide a novel generalization analysis framework for the Lion optimizer
	based on the mathematical induction.
	Moreover, we discover a meaningful finding that
	the smallest absolute value of non-zero element in gradient estimators $\{c_t\}_{t\geq1}$
	affects generalization error of the Lion optimizer.
	
	When $\beta_1=\beta_2=\lambda=0$ in Algorithm~\ref{alg:1},
	the Lion algorithm will reduce to SignSGD algorithm. Thus, from Theorem~\ref{th:1},
	the SignSGD algorithm~\citep{bernstein2018signsgd} also has a generalization error of $O(\frac{1}{N\tau^T})$.
\end{remark}

\subsection{Generalization Error of our CLion Optimizer}
In this subsection, we provide a generalization error for our CLion optimizer.

\begin{theorem} \label{th:2}
	Assume the sequence $\{w_t\}_{t=1}^T$ is generated
	from Algorithm~\ref{alg:2} on dataset $S$. Under the Assumptions~\ref{ass:s1},~\ref{ass:g},~\ref{ass:v}, without loss of generality, let $\nu\geq 1$, $\lambda=O(1)$ with $0<\lambda\leq \frac{1}{\eta}$, $\beta_1=O(1)$ with $\beta_1\in (0,1)$, $\beta_2=O(1)$ with $\beta_2\in (0,1)$, $\sigma=O(1)$, $L=O(1)$ and $G=O(1)$.
	When the iteration number is relatively small (i.e., $T=O(1)$) set
	$\eta=\frac{1}{\sqrt{d}}$, otherwise set $\eta=\frac{1}{\sqrt{d}T}$, we have
	\begin{align}
		|\E [F(w_T) - F_S(w_T)]| \leq O(\frac{1}{N}).
	\end{align}
\end{theorem}

\begin{proof}
	This proof basically follows proof of the Theorem~\ref{th:1}. All detailed proof is provided in Appendix~\ref{ga:clion}.
	% We first prove
	% \begin{align}
		% \|w_{t} - w_{t}^{(i)}\| \leq (1-\eta\lambda)\|w_{t-1} -w_{t-1}^{(i)}\|  + 2\eta\sqrt{d}\|c_t-c_t^{(i)}\|.
		% \end{align}
	%When the iteration number is relatively small, i.e., $T=O(1)$,
	%and set $\eta=\frac{1}{\sqrt{d}}$, $\lambda=O(1)$ with $0<\lambda\leq \frac{1}{\eta}$, $\beta_1=O(1)$ with $\beta_1\in (0,1)$, $\beta_2=O(1)$ with $\beta_2\in (0,1)$, $\sigma=O(1)$ and
	%$L=O(1)$. Following the above proof of Theorem~\ref{th:1}, we can obtain
	%\begin{align}
	%  \E \|w_{T} - w_{T}^{(i)}\| \leq O(\frac{1}{N}).
	%\end{align}
	%
	
\end{proof}

\begin{remark}
	From Theorem~\ref{th:2}, under the same conditions, our CLion optimizer has a lower generalization error of $O(\frac{1}{N})$ than $O(\frac{1}{\tau^T N})$ of the Lion optimizer, since the value $\tau$ generally is very small.
\end{remark}

\section{Convergence Analysis}
In this section, we provide the convergence properties of our CLion optimizer.
All detailed proofs of convergence analysis of
our CLion optimizer is provided in Appendix~\ref{ca:clion}.

\begin{assumption}[\textbf{Smoothness of Population Function}]\label{ass:s2}
	Population function $F(w)$ is $L$-Lipschitz smooth, if for any $w_1, w_2\in \R^{d}$, we have
	\begin{align}
		\|\nabla F(w_1)-\nabla F(w_2)\| \leq L\|w_1-w_2\|.
	\end{align}
\end{assumption}

\begin{assumption} \label{ass:f}
	The function $F(w)$ has a lower bounded, i.e.,  $F^* = \inf_{w\in \R^d} F(w)>-\infty$.
\end{assumption}

Assumption~\ref{ass:s2} shows smoothness of population function $F(w)$, which is widely used in nonconvex optimization~\citep{bottou2018optimization,ghadimi2013stochastic}. Assumption~\ref{ass:f} guarantees feasibility of the above problem~(\ref{eq:p1}), which also is widely used in optimization~\citep{bottou2018optimization,ghadimi2013stochastic}.

\begin{lemma} \label{lem:2}
	Assume the sequence $\{w_t\}_{t=1}^T$ is generated from Algorithm~\ref{alg:2}, let $\|w_{0}\|\leq \eta\hat{G}$ and $\lambda \leq \frac{1}{2\eta\hat{G} T^{\alpha}}$ with $\alpha>1$, we have
	\begin{align}
		\|w_t\|\leq (t+1)\eta\hat{G}, \quad  \|w_t - w_{t-1}\|\leq 2\eta \hat{G},
	\end{align}
	where $\hat{G}=\max(G,\sqrt{d})$.
\end{lemma}

\begin{lemma} \label{lem:3}
	Assume the sequence $\{c_t\}_{t=1}^T$ is generated from Algorithm~\ref{alg:2}, let $\|w_{0}\|\leq \eta\hat{G}$ and $\lambda \leq \frac{1}{2\eta\hat{G} T^{\alpha}}$ with $\alpha>1$, we have
	\begin{align}
		\frac{1}{T} \sum_{t=1}^T\E\|c_{t+1}-\nabla F(w_{t})\| \leq \frac{\sqrt{2(\sigma^2+G^2)}}{\sqrt{(1-\beta_2)T}}+ \frac{2\sqrt{2}L \hat{G}\eta}{1-\beta_2}  + \frac{\sqrt{2}|\beta_1-\beta_2|}{\sqrt{1-\beta_2}}\sigma+\frac{1-\beta_1}{\sqrt{1-\beta_2}}\sigma,
	\end{align}
	where $\hat{G}=\max(G,\sqrt{d})$.
\end{lemma}

\begin{theorem} \label{th:3}
	Assume the sequence $\{w_t\}_{t=1}^T$ is generated
	from Algorithm~\ref{alg:2}. Under the Assumptions~\ref{ass:s2},~\ref{ass:g},~\ref{ass:v},~\ref{ass:f}, and let $0<\lambda \leq \frac{1}{2\eta\hat{G} T^{\alpha}}$, $\eta=O(\frac{1}{T^{3/4}})$, $\beta_1=1-O(\frac{1}{\sqrt{T}})$, $\beta_2=1-O(\frac{1}{\sqrt{T}})$, $|\beta_1-\beta_2|=O(\frac{1}{\sqrt{T}})$ and $0<\nu_0 \leq \nu$,
	and further set $\alpha=\frac{5}{4}$ and $\nu_0 \geq O(\frac{1}{\sqrt{d}})$, we can obtain
	\begin{align}
		\frac{1}{T}\sum_{t=1}^T\E\|\nabla F(w_{t})\|_1 \leq O(\frac{\sqrt{d}}{T^{1/4}}).
	\end{align}
\end{theorem}

\begin{remark}
	From the above Theorem~\ref{th:3}, our CLion has a fast convergence rate of $O(\frac{\sqrt{d}}{T^{1/4}})$ on $\ell_1$-norm of gradient, which
	has the same convergence rate as the Lion optimizer~\citep{dong2024convergence,jiang2025convergence}.
\end{remark}

\begin{figure}[ht]
	\centering
	\subfigure[Train Loss]{\includegraphics[width=0.24\textwidth]{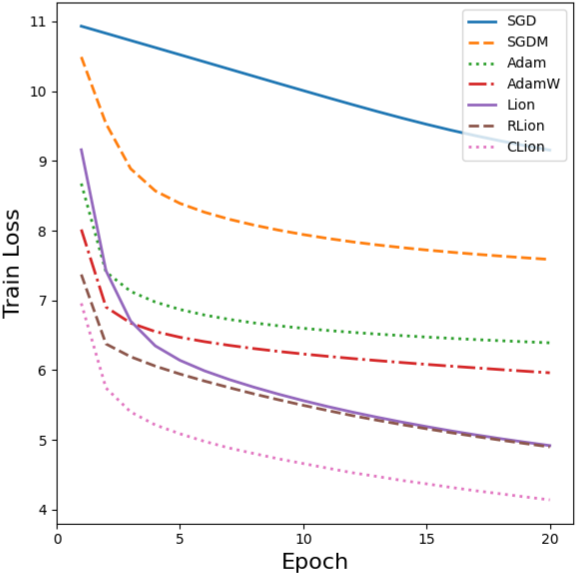}}
	\hfill
	\subfigure[Train Perplexity]{\includegraphics[width=0.24\textwidth]{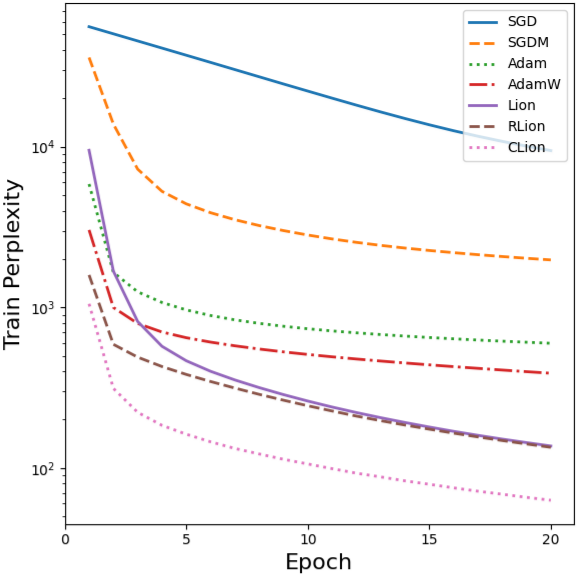}}
	\hfill
	\subfigure[Test Loss]{\includegraphics[width=0.24\textwidth]{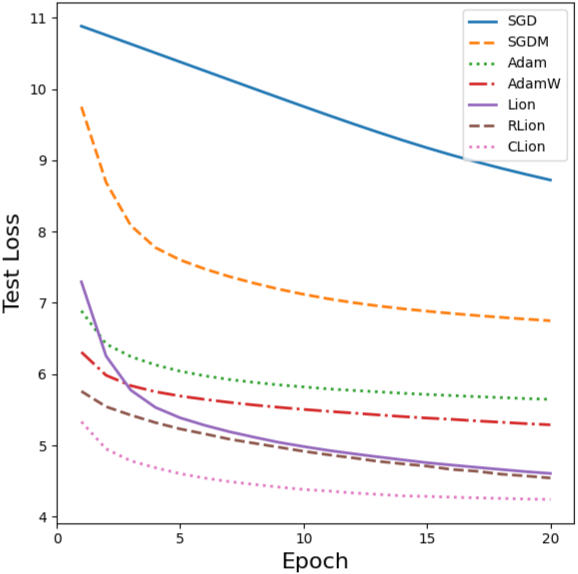}}
	\hfill
	\subfigure[Test Perplexity]{\includegraphics[width=0.24\textwidth]{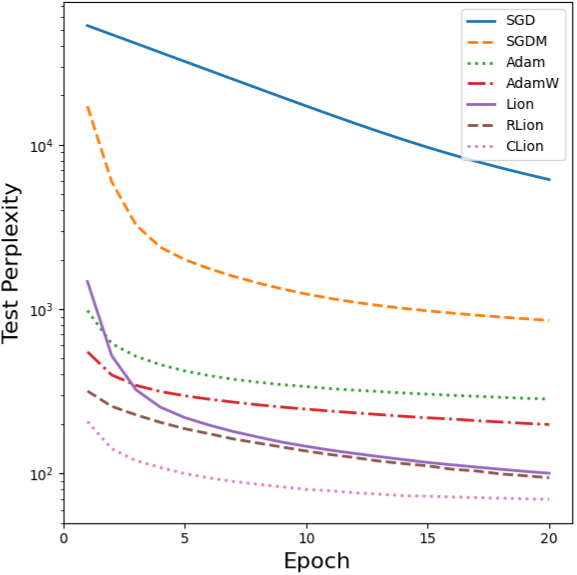}}
	\hfill
	\caption{Language modeling at \emph{Wikitext-2} dataset.}
	\label{fig:2}
\end{figure}

\begin{figure}[ht]
	\centering
	\subfigure[Train Loss]{\includegraphics[width=0.24\textwidth]{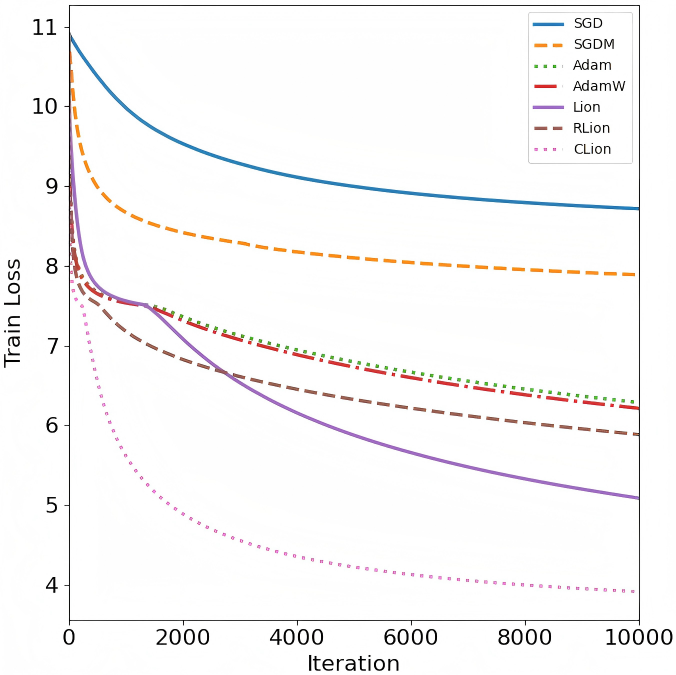}}
	\hfill
	\subfigure[Train Perplexity]{\includegraphics[width=0.24\textwidth]{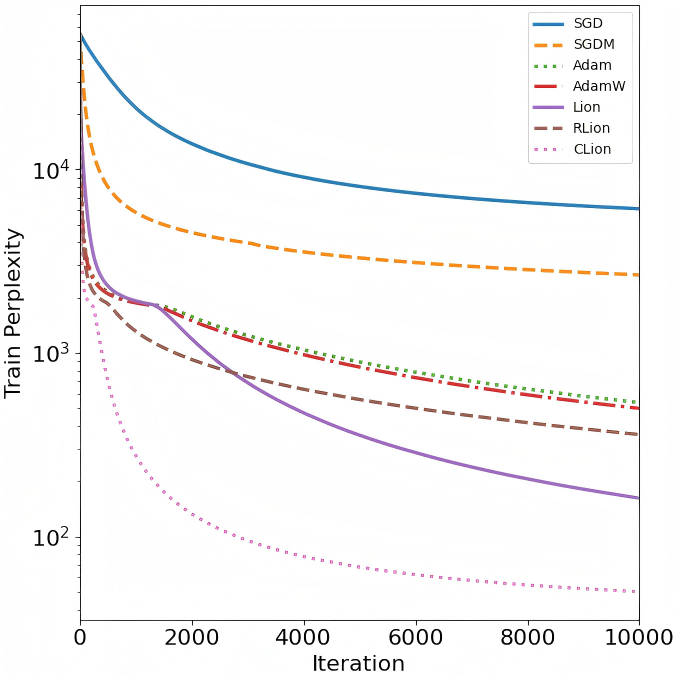}}
	\hfill
	\subfigure[Test Loss]{\includegraphics[width=0.24\textwidth]{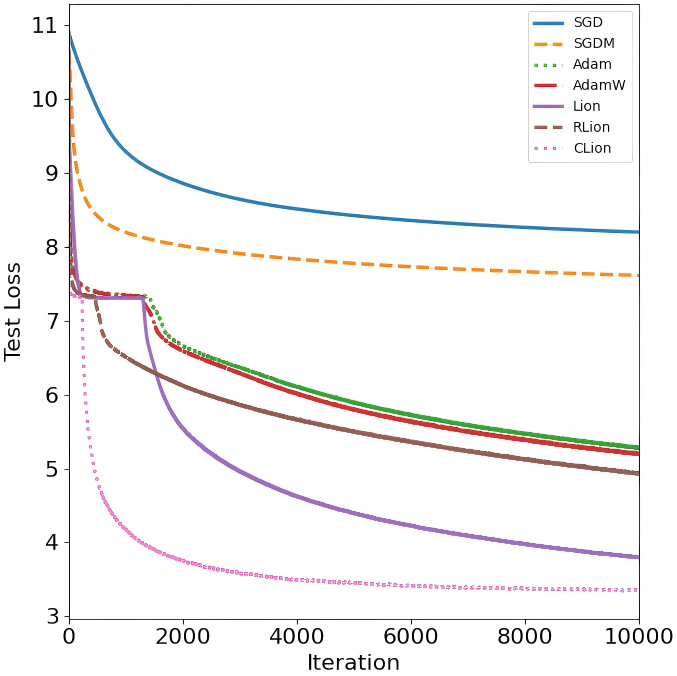}}
	\hfill
	\subfigure[Test Perplexity]{\includegraphics[width=0.24\textwidth]{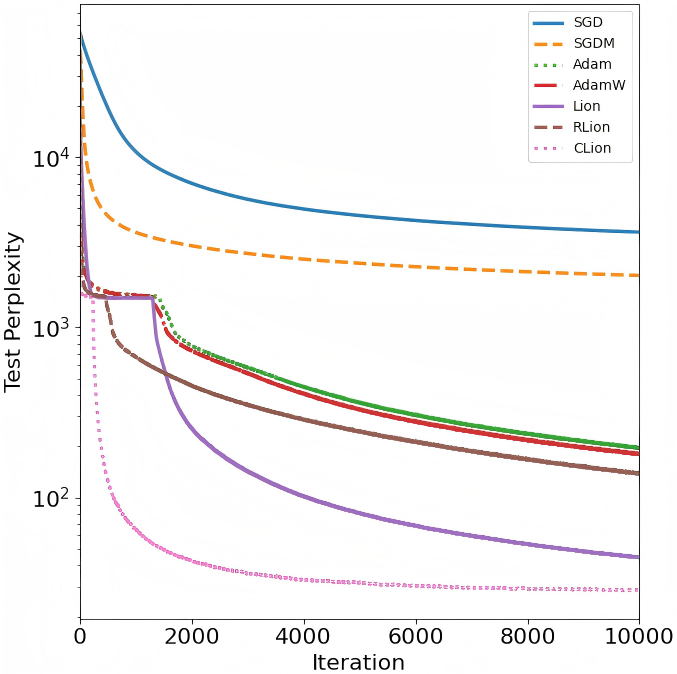}}
	\hfill
	\caption{Language modeling at \emph{Wikitext-103} dataset.}
	\label{fig:3}
\end{figure}

\section{Numerical Experiments}
In the section, we conduct some numerical experiments to demonstrate efficiency of our CLion optimizer
on image classification and language modeling tasks. In the experiment, we compare our CLion optimizer with
some representative optimizers including the SGD, SGDM, Adam~\citep{kingma2014adam},  AdamW~\citep{loshchilov2017decoupled}, Lion~\citep{chen2023symbolic} and RLion~\citep{rong2025refined}.

\subsection{Language Modeling }
In this experiment, given some training samples $\{z^i\}_{i=1}^N$, we conduct language modeling task by solving the following nonconvex problem
\begin{align}
	\min_{w \in \mathbb{R}^d} -\frac{1}{N} \sum_{i=1}^{N} \sum_{t=1}^{m_i} \log \big(p(z^i_{t} | z^i_{1:t-1}; w)\big),
\end{align}
where each sample $z^i$ includes $m_i$ tokens,
and $p(z^i_{t} | z^i_{1:t-1}; w)$ denotes a probability function
of token $z^i_{t}$ given the tokens $z^i_{1:t-1}$, and $w \in \mathbb{R}^d$ denotes parameters of
the language model.

In the experiment, we first evaluate the language modeling task
on the WikiText-2~\citep{merity2016pointer} dataset.
This language model is modeled as a 7-layer Transformer~\citep{vaswani2017attention} encoder with 768-dimensional embeddings and 8 attention heads per layer, which employs a feed-forward network dimension of 1024 and uses sinusoidal positional encodings. Meanwhile, it uses a dropout rate of 0.1 throughout the network. The final output layer projects the representations back to vocabulary size for token prediction.
Then we evaluate the language modeling task
on the WikiText-103~\citep{merity2016pointer} dataset.
This language model is modeled as a 21-layer Transformer encoder with 768-dimensional embeddings and 8 attention heads per layer, which employs a feed-forward network dimension of 2048 and uses sinusoidal positional encodings. Meanwhile, it uses a dropout rate of 0.15 throughout the network.

 In the experiment, for all hyper-parameters, we do grid search and report the best one for each optimizer. 
When training 7-layer Transformer model at WikiText2 dataset, we set batch size be 10 for all algorithms.
We set the learning rate $2\times10^{-5}$ for SGD and SGDM,
and set momentum parameter $\beta=0.9$ for SGDM. Adam and AdamW use
the basic learning rate $2\times10^{-5}$, $\varepsilon=10^{-8}$,
the first-order momentum parameter $\beta_1=0.9$, and
the second-order momentum parameter $\beta_2=0.999$. Meanwhile, AdamW uses the weight decay parameter
$\lambda=10^{-8}$.
We set the learning rate $2\times 10^{-6}$, $\beta_1=0.9$, $\beta_2=0.99$ and $\lambda=10^{-8}$ for the Lion.
RLion uses the learning rate $2\times 10^{-5}$, $\beta_1=0.9$, $\beta_2=0.99$, $\lambda=10^{-8}$ and $\alpha=10^4$.
Our CLion uses the learning rate $2\times 10^{-5}$, $\beta_1=0.9$,
$\beta_2=0.99$, $\lambda=10^{-8}$ and the threshold $\nu=10^{-13}$.

When training 21-layer Transformer model at WikiText103 dataset,
we set the batch size be 10 for all algorithms.
We set the learning rate $2\times10^{-5}$ for SGD and SGDM,
and set momentum parameter $\beta=0.9$ for SGDM. Adam and AdamW use
the basic learning rate $2\times10^{-5}$, $\varepsilon=10^{-8}$,
the first-order momentum parameter $\beta_1=0.9$, and
the second-order momentum parameter $\beta_2=0.99$. Meanwhile, AdamW uses the weight decay parameter
$\lambda=10^{-6}$.
We set the learning rate $2\times 10^{-5}$, $\beta_1=0.9$, $\beta_2=0.99$ and $\lambda=10^{-9}$ for the Lion.
RLion uses the learning rate $2\times 10^{-5}$, $\beta_1=0.9$, $\beta_2=0.99$, $\lambda=10^{-9}$ and $\alpha=10^4$.
Our CLion uses the learning rate $2\times 10^{-5}$, $\beta_1=0.9$,
$\beta_2=0.99$, $\lambda=10^{-4}$ and the threshold $\nu=10^{-15}$.

Figures~\ref{fig:2} and~\ref{fig:3} show that our CLion optimizer outperforms other optimizers such as Lion and RLion on training and test errors, which demonstrate efficiency of our CLion optimizer.
Meanwhile, these results also verify that
our CLion optimizer has better generalization than the Lion.

\begin{figure}[ht]
	\centering
	\subfigure[Train Loss]{\includegraphics[width=0.24\textwidth]{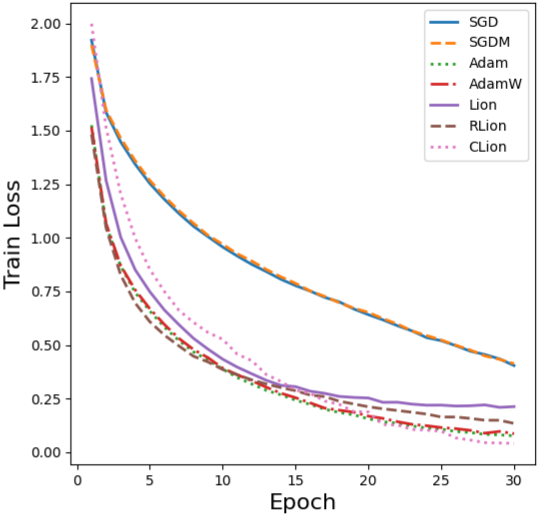}}
	\hfill
	\subfigure[Train Accuracy]{\includegraphics[width=0.24\textwidth]{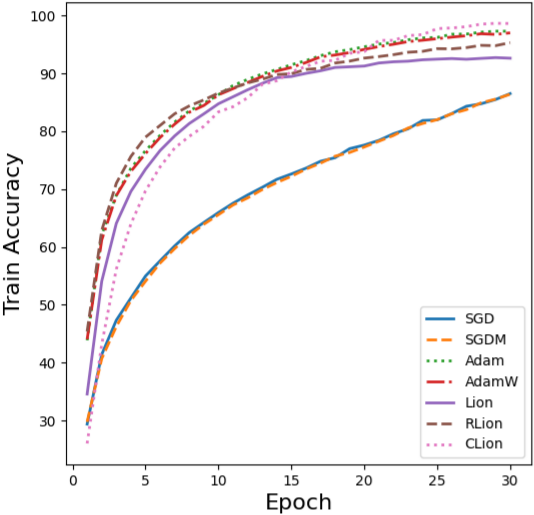}}
	\hfill
	\subfigure[Test Loss]{\includegraphics[width=0.24\textwidth]{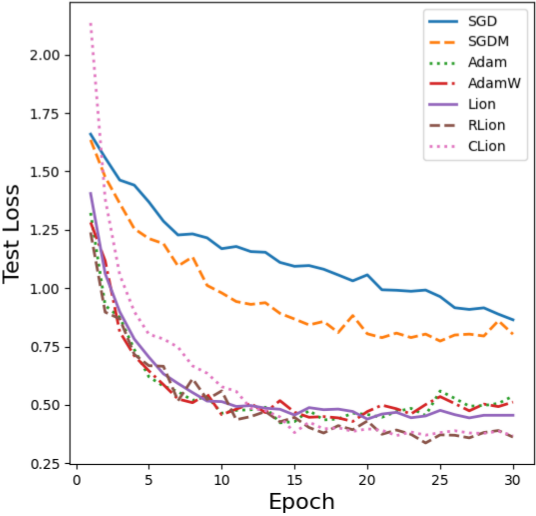}}
	\hfill
	\subfigure[Test Accuracy]{\includegraphics[width=0.24\textwidth]{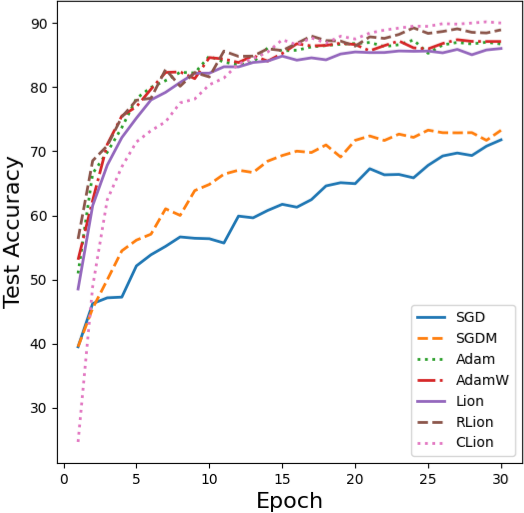}}
	\hfill
	\caption{Image classification at \emph{Cifar-10} dataset.}
	\label{fig:4}
\end{figure}

\begin{figure}[ht]
	\centering
	\subfigure[Train Loss]{\includegraphics[width=0.24\textwidth]{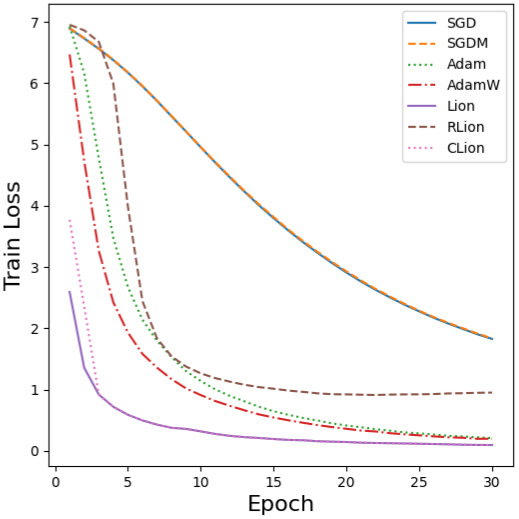}}
	\hfill
	\subfigure[Train Accuracy]{\includegraphics[width=0.24\textwidth]{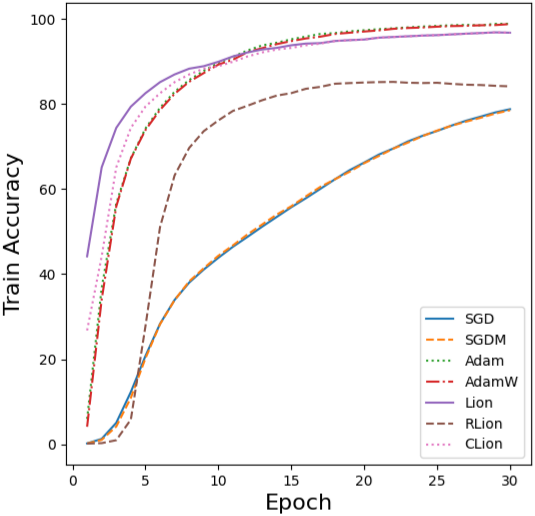}}
	\hfill
	\subfigure[Test Loss]{\includegraphics[width=0.24\textwidth]{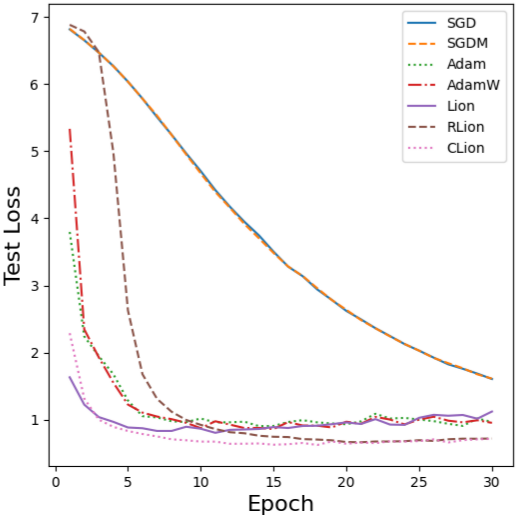}}
	\hfill
	\subfigure[Test Accuracy]{\includegraphics[width=0.24\textwidth]{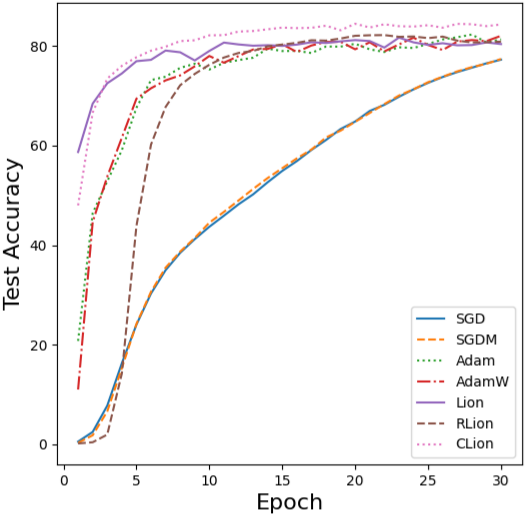}}
	\hfill
	\caption{Image classification at \emph{tiny-ImageNet} dataset.}
	\label{fig:5}
\end{figure}

\subsection{Image Classification}
In this experiment, we train two deep learning models
to image classification.
Given training samples $\{x_i,y_i\}_{i=1}^N$, where $x_i$ denotes features and
$y_i$ denotes label, we train deep learning model by solving the following problem
\begin{align}
	\min_{w\in \mathbb{R}^d} \frac{1}{N}\sum_{i=1}^N \ell \big(\chi(x_i;w),y_i\big),
\end{align}
where $\chi(\cdot;w)$ represents deep learning model, and
$\ell(\cdot,\cdot)$ denotes a cross-entropy loss function.

In the experiment, we first evaluate ResNet18~\citep{he2016deep}
on the Cifar-10~\citep{krizhevsky2009learning} dataset, where
the training and test datesets contain 50000 and 10000 samples, respectively.
Then we evaluate the ResNet34~\citep{he2016deep} on tiny-ImageNet~\citep{le2015tiny} dataset, where
the training and test datesets contain 80000 and 20000 samples, respectively.

In the experiment, for all hyper-parameters, we do grid search and report the best one for each optimizer. 
When training Resnet18 at CIFAR-10 dataset, we set the batch size be 64 for all algorithms.
We set the learning rate $3\times10^{-4}$ for SGD, and set
the learning rate $10^{-4}$ and momentum parameter $\beta=0.9$ for SGDM.
Adam and AdamW
use the basic learning rate $10^{-4}$, the tuning parameter $\varepsilon=10^{-8}$, the first-order momentum parameter $\beta_1=0.9$, and
the second-order momentum parameter $\beta_2=0.99$. Meanwhile, AdamW uses the weight decay parameter
$\lambda=10^{-6}$.
We set the learning rate $10^{-4}$, $\beta_1=0.9$, $\beta_2=0.99$ and $\lambda=10^{-2}$ for the Lion.
RLion uses the learning rate $10^{-4}$, $\beta_1=0.9$, $\beta_2=0.99$, $\lambda=10^{-2}$ and $\alpha=10$.
Our CLion uses the learning rate $10^{-4}$, $\beta_1=0.9$,
$\beta_2=0.99$, $\lambda=10^{-2}$ and the threshold $\nu=10^{-6}$.

When training ResNet34 at Tiny-ImageNet dataset,
we set the batch size be 64 for all algorithms.
We set the learning rate $10^{-3}$ for SGD, and set
the learning rate $10^{-3}$ and momentum parameter $\beta=0.9$ for SGDM.
Adam and AdamW
use the basic learning rate $10^{-3}$, the tuning parameter $\varepsilon=10^{-8}$, the first-order momentum parameter $\beta_1=0.9$, and
the second-order momentum parameter $\beta_2=0.99$. Meanwhile, AdamW uses the weight decay parameter
$\lambda=10^{-6}$.
We set the learning rate $10^{-4}$, $\beta_1=0.9$, $\beta_2=0.99$ and $\lambda=10^{-2}$ for the Lion.
RLion uses the learning rate $10^{-3}$, $\beta_1=0.9$, $\beta_2=0.99$, $\lambda=10^{-2}$ and $\alpha=10$.
Our CLion uses the learning rate $10^{-3}$, $\beta_1=0.9$,
$\beta_2=0.99$, $\lambda=10^{-6}$ and the threshold $\nu=10^{-8}$.

From Figures~\ref{fig:4} and~\ref{fig:5}, we could find that our CLion optimizer basically outperforms other optimizers such as Lion and RLion on test accuracy. These results also verify that
our CLion optimizer has a good generalization.

\section{Conclusions}
In this paper, we first studied generalization property of the Lion optimizer via algorithmic stability, and  discovered a useful finding that
the smallest absolute value of non-zero element in gradient estimators affects generalization error of the Lion optimizer. To improve generalization of the Lion, we proposed a novel efficient cautious Lion (i.e., CLion) optimizer by cautiously using sign function. Moreover, we proved that our CLion optimizer has a lower generalization error $O(\frac{1}{N})$ than $O(\frac{1}{N\tau^T})$ of the Lion, since the value $\tau$ generally is very small. Meanwhile, we proved that our CLion has a fast convergence rate as the Lion optimizer.
In addition, we obtain an interesting byproduct that the SignSGD algorithm has a generalization error of $O(\frac{1}{N\tau^T})$ under the nonconvex setting.

\small

\bibliographystyle{plainnat}

\bibliography{CLion}

\begin{thebibliography}{35}
\providecommand{\natexlab}[1]{#1}
\providecommand{\url}[1]{\texttt{#1}}
\expandafter\ifx\csname urlstyle\endcsname\relax
  \providecommand{\doi}[1]{doi: #1}\else
  \providecommand{\doi}{doi: \begingroup \urlstyle{rm}\Url}\fi

\bibitem[Bernstein et~al.(2018)Bernstein, Wang, Azizzadenesheli, and
  Anandkumar]{bernstein2018signsgd}
Jeremy Bernstein, Yu-Xiang Wang, Kamyar Azizzadenesheli, and Animashree
  Anandkumar.
\newblock signsgd: Compressed optimisation for non-convex problems.
\newblock In \emph{International conference on machine learning}, pages
  560--569. PMLR, 2018.

\bibitem[Bottou et~al.(2018)Bottou, Curtis, and
  Nocedal]{bottou2018optimization}
L{\'e}on Bottou, Frank~E Curtis, and Jorge Nocedal.
\newblock Optimization methods for large-scale machine learning.
\newblock \emph{SIAM review}, 60\penalty0 (2):\penalty0 223--311, 2018.

\bibitem[Chen et~al.(2023{\natexlab{a}})Chen, Liu, Liang, and
  Liu]{chen2023lion}
Lizhang Chen, Bo~Liu, Kaizhao Liang, and Qiang Liu.
\newblock Lion secretly solves constrained optimization: As lyapunov predicts.
\newblock \emph{arXiv preprint arXiv:2310.05898}, 2023{\natexlab{a}}.

\bibitem[Chen et~al.(2023{\natexlab{b}})Chen, Liang, Huang, Real, Wang, Pham,
  Dong, Luong, Hsieh, Lu, et~al.]{chen2023symbolic}
Xiangning Chen, Chen Liang, Da~Huang, Esteban Real, Kaiyuan Wang, Hieu Pham,
  Xuanyi Dong, Thang Luong, Cho-Jui Hsieh, Yifeng Lu, et~al.
\newblock Symbolic discovery of optimization algorithms.
\newblock \emph{Advances in neural information processing systems},
  36:\penalty0 49205--49233, 2023{\natexlab{b}}.

\bibitem[Chen et~al.(2022)Chen, Chen, Cheng, Chen, Awadallah, and
  Wang]{chen2022scalable}
Xuxi Chen, Tianlong Chen, Yu~Cheng, Weizhu Chen, Ahmed Awadallah, and Zhangyang
  Wang.
\newblock Scalable learning to optimize: A learned optimizer can train big
  models.
\newblock In \emph{European Conference on Computer Vision}, pages 389--405.
  Springer, 2022.

\bibitem[Cutkosky and Orabona(2019)]{cutkosky2019momentum}
Ashok Cutkosky and Francesco Orabona.
\newblock Momentum-based variance reduction in non-convex sgd.
\newblock \emph{Advances in neural information processing systems}, 32, 2019.

\bibitem[Dong et~al.(2024)Dong, Li, and Lin]{dong2024convergence}
Yiming Dong, Huan Li, and Zhouchen Lin.
\newblock Convergence rate analysis of lion.
\newblock \emph{arXiv preprint arXiv:2411.07724}, 2024.

\bibitem[Frank et~al.(1956)Frank, Wolfe, et~al.]{frank1956algorithm}
Marguerite Frank, Philip Wolfe, et~al.
\newblock An algorithm for quadratic programming.
\newblock \emph{Naval research logistics quarterly}, 3\penalty0 (1-2):\penalty0
  95--110, 1956.

\bibitem[Ghadimi and Lan(2013)]{ghadimi2013stochastic}
Saeed Ghadimi and Guanghui Lan.
\newblock Stochastic first-and zeroth-order methods for nonconvex stochastic
  programming.
\newblock \emph{SIAM journal on optimization}, 23\penalty0 (4):\penalty0
  2341--2368, 2013.

\bibitem[Hardt et~al.(2016)Hardt, Recht, and Singer]{hardt2016train}
Moritz Hardt, Ben Recht, and Yoram Singer.
\newblock Train faster, generalize better: Stability of stochastic gradient
  descent.
\newblock In \emph{International conference on machine learning}, pages
  1225--1234. PMLR, 2016.

\bibitem[Harrison et~al.(2022)Harrison, Metz, and
  Sohl-Dickstein]{harrison2022closer}
James Harrison, Luke Metz, and Jascha Sohl-Dickstein.
\newblock A closer look at learned optimization: Stability, robustness, and
  inductive biases.
\newblock \emph{Advances in neural information processing systems},
  35:\penalty0 3758--3773, 2022.

\bibitem[He et~al.(2016)He, Zhang, Ren, and Sun]{he2016deep}
Kaiming He, Xiangyu Zhang, Shaoqing Ren, and Jian Sun.
\newblock Deep residual learning for image recognition.
\newblock In \emph{Proceedings of the IEEE conference on computer vision and
  pattern recognition}, pages 770--778, 2016.

\bibitem[Jiang and Zhang(2025)]{jiang2025convergence}
Wei Jiang and Lijun Zhang.
\newblock Convergence analysis of the lion optimizer in centralized and
  distributed settings.
\newblock \emph{arXiv preprint arXiv:2508.12327}, 2025.

\bibitem[Kingma and Ba(2014)]{kingma2014adam}
Diederik~P Kingma and Jimmy Ba.
\newblock Adam: A method for stochastic optimization.
\newblock \emph{arXiv preprint arXiv:1412.6980}, 2014.

\bibitem[Krizhevsky et~al.(2009)Krizhevsky, Hinton,
  et~al.]{krizhevsky2009learning}
Alex Krizhevsky, Geoffrey Hinton, et~al.
\newblock Learning multiple layers of features from tiny images.
\newblock 2009.

\bibitem[Le and Yang(2015)]{le2015tiny}
Ya~Le and Xuan Yang.
\newblock Tiny imagenet visual recognition challenge.
\newblock \emph{CS 231N}, 7\penalty0 (7):\penalty0 3, 2015.

\bibitem[Lei(2023)]{lei2023stability}
Yunwen Lei.
\newblock Stability and generalization of stochastic optimization with
  nonconvex and nonsmooth problems.
\newblock In \emph{The Thirty Sixth Annual Conference on Learning Theory},
  pages 191--227. PMLR, 2023.

\bibitem[Lei and Ying(2020)]{lei2020fine}
Yunwen Lei and Yiming Ying.
\newblock Fine-grained analysis of stability and generalization for stochastic
  gradient descent.
\newblock In \emph{International Conference on Machine Learning}, pages
  5809--5819. PMLR, 2020.

\bibitem[Liu et~al.(2024)Liu, Wu, Chen, Liang, Zhu, Liang, Krishnamoorthi, and
  Liu]{liu2024communication}
Bo~Liu, Lemeng Wu, Lizhang Chen, Kaizhao Liang, Jiaxu Zhu, Chen Liang,
  Raghuraman Krishnamoorthi, and Qiang Liu.
\newblock Communication efficient distributed training with distributed lion.
\newblock \emph{Advances in Neural Information Processing Systems},
  37:\penalty0 18388--18415, 2024.

\bibitem[Loshchilov and Hutter(2017)]{loshchilov2017decoupled}
Ilya Loshchilov and Frank Hutter.
\newblock Decoupled weight decay regularization.
\newblock \emph{arXiv preprint arXiv:1711.05101}, 2017.

\bibitem[Merity et~al.(2016)Merity, Xiong, Bradbury, and
  Socher]{merity2016pointer}
Stephen Merity, Caiming Xiong, James Bradbury, and Richard Socher.
\newblock Pointer sentinel mixture models.
\newblock \emph{arXiv preprint arXiv:1609.07843}, 2016.

\bibitem[Nesterov et~al.(2018)]{nesterov2018lectures}
Yurii Nesterov et~al.
\newblock \emph{Lectures on convex optimization}, volume 137.
\newblock Springer, 2018.

\bibitem[Ramezani-Kebrya et~al.(2024)Ramezani-Kebrya, Antonakopoulos, Cevher,
  Khisti, and Liang]{ramezani2024generalization}
Ali Ramezani-Kebrya, Kimon Antonakopoulos, Volkan Cevher, Ashish Khisti, and
  Ben Liang.
\newblock On the generalization of stochastic gradient descent with momentum.
\newblock \emph{Journal of Machine Learning Research}, 25\penalty0
  (22):\penalty0 1--56, 2024.

\bibitem[Reddi et~al.(2016)Reddi, Sra, P{\'o}czos, and
  Smola]{reddi2016stochastic}
Sashank~J Reddi, Suvrit Sra, Barnab{\'a}s P{\'o}czos, and Alex Smola.
\newblock Stochastic frank-wolfe methods for nonconvex optimization.
\newblock In \emph{2016 54th annual Allerton conference on communication,
  control, and computing (Allerton)}, pages 1244--1251. IEEE, 2016.

\bibitem[Robbins and Monro(1951)]{robbins1951stochastic}
Herbert Robbins and Sutton Monro.
\newblock A stochastic approximation method.
\newblock \emph{The annals of mathematical statistics}, pages 400--407, 1951.

\bibitem[Rong et~al.(2025)Rong, Ma, Zhang, Cao, and Kou]{rong2025refined}
Jian Rong, Chenhao Ma, Qinghui Zhang, Yong Cao, and Weili Kou.
\newblock A refined lion optimizer for deep learning.
\newblock \emph{Scientific Reports}, 15\penalty0 (1):\penalty0 23082, 2025.

\bibitem[Sfyraki and Wang(2025)]{sfyraki2025lions}
Maria-Eleni Sfyraki and Jun-Kun Wang.
\newblock Lions and muons: Optimization via stochastic frank-wolfe.
\newblock \emph{arXiv preprint arXiv:2506.04192}, 2025.

\bibitem[Shalev-Shwartz et~al.(2010)Shalev-Shwartz, Shamir, Srebro, and
  Sridharan]{shalev2010learnability}
Shai Shalev-Shwartz, Ohad Shamir, Nathan Srebro, and Karthik Sridharan.
\newblock Learnability, stability and uniform convergence.
\newblock \emph{The Journal of Machine Learning Research}, 11:\penalty0
  2635--2670, 2010.

\bibitem[Sutskever et~al.(2013)Sutskever, Martens, Dahl, and
  Hinton]{sutskever2013importance}
Ilya Sutskever, James Martens, George Dahl, and Geoffrey Hinton.
\newblock On the importance of initialization and momentum in deep learning.
\newblock In \emph{International conference on machine learning}, pages
  1139--1147. pmlr, 2013.

\bibitem[Tran-Dinh et~al.(2022)Tran-Dinh, Pham, Phan, and
  Nguyen]{tran2022hybrid}
Quoc Tran-Dinh, Nhan~H Pham, Dzung~T Phan, and Lam~M Nguyen.
\newblock A hybrid stochastic optimization framework for composite nonconvex
  optimization.
\newblock \emph{Mathematical Programming}, 191\penalty0 (2):\penalty0
  1005--1071, 2022.

\bibitem[Vaswani et~al.(2017)Vaswani, Shazeer, Parmar, Uszkoreit, Jones, Gomez,
  Kaiser, and Polosukhin]{vaswani2017attention}
Ashish Vaswani, Noam Shazeer, Niki Parmar, Jakob Uszkoreit, Llion Jones,
  Aidan~N Gomez, {\L}ukasz Kaiser, and Illia Polosukhin.
\newblock Attention is all you need.
\newblock \emph{Advances in neural information processing systems}, 30, 2017.

\bibitem[Wu et~al.(2020)Wu, Xu, Dai, Wan, Zhang, Yan, Tomizuka, Gonzalez,
  Keutzer, and Vajda]{wu2020visual}
Bichen Wu, Chenfeng Xu, Xiaoliang Dai, Alvin Wan, Peizhao Zhang, Zhicheng Yan,
  Masayoshi Tomizuka, Joseph Gonzalez, Kurt Keutzer, and Peter Vajda.
\newblock Visual transformers: Token-based image representation and processing
  for computer vision, 2020.

\bibitem[Yu et~al.(2026)Yu, Tao, Wan, Luo, and Zhang]{yu2026sign}
Dingzhi Yu, Hongyi Tao, Yuanyu Wan, Luo Luo, and Lijun Zhang.
\newblock Sign-based optimizers are effective under heavy-tailed noise.
\newblock \emph{arXiv preprint arXiv:2602.07425}, 2026.

\bibitem[Yuan et~al.(2024)Yuan, Liu, Wu, Zhou, and Gu]{yuan2024mars}
Huizhuo Yuan, Yifeng Liu, Shuang Wu, Xun Zhou, and Quanquan Gu.
\newblock Mars: Unleashing the power of variance reduction for training large
  models.
\newblock \emph{arXiv preprint arXiv:2411.10438}, 2024.

\bibitem[Zhang(2023)]{zhang2023mathematical}
Tong Zhang.
\newblock \emph{Mathematical analysis of machine learning algorithms}.
\newblock Cambridge University Press, 2023.

\end{thebibliography}

\newpage

\appendix

\section{Generalization Analysis of Lion Optimizer}
\label{ga:lion}

\begin{lemma} (Restatement of Lemma~\ref{lem:1})
	Assume the sequences $\{c_t\}_{t=1}^T$ and $\{c_t^{(i)}\}_{t=1}^T$ are generated from Algorithm~\ref{alg:1}
	based on the dataset $S$ and
	$S^{(i)}$, respectively, we have
	\begin{align}
		\big\|\sign(c_t)-\sign(c_t^{(i)})\big\|  \leq \frac{2\sqrt{d}}{\tau}\|c_t-c_t^{(i)}\|,
	\end{align}
	where $\tau=\min_{t\geq 1}(\min_{j\in S_t}(|(c_t)_j|))>0$ with $S_t=\{j|\ |(c_t)_j|\neq 0, j=1,2\cdots,d\}$.
\end{lemma}

\begin{proof}
	When $(c_t)_j$ and $(c_t^{(i)})_j$ have the same sign, i.e., $\sign((c_t)_j)=\sign((c_t^{(i)})_j)$, we have
	\begin{align}
		|(c_t)_j-(c_t^{(i)})_j| \geq 0=|\sign(c_t)_j-\sign(c_t^{(i)})_j|.
	\end{align}
	
	Let $\tau=\min_t(\min_{j\in S_t}(|(c_t)_j|))>0$ with $S_t=\{j|\ |(c_t)_j|\neq 0, j=1,2\cdots,d\}$, where
	$c_t$ is generated from Algorithm~\ref{alg:1} for any $t\geq 1$. Since $c^{(i)}_t$ also is generated from Algorithm~\ref{alg:1}, we have $\tau=\min_t(\min_{j\in S_t}(|(c^{(i)}_t)_j|))>0$ with $S_t=\{j|\ |(c^{(i)}_t)_j|\neq 0, j=1,2\cdots,d\}$.
	
	When $(c_t)_j$ and $(c_t^{(i)})_j$ have different sign, i.e., $|\sign((c_t)_j)-\sign((c_t^{(i)})_j)|=2$ or
	$|\sign((c_t)_j)-\sign((c_t^{(i)})_j)|=1$,
	\begin{align}
		\frac{2}{\tau}|(c_t)_j-(c_t^{(i)})_j| = \frac{2}{\tau}(|(c_t)_j|+|(c_t^{(i)})_j|) \mathop{\geq}^{(i)} 2 \geq |\sign(c_t)_j-\sign(c_t^{(i)})_j|,
	\end{align}
	where the above inequality $(i)$ is due to $|(c_t)_j|+|(c_t^{(i)})_j|\geq \tau$.
	
	Thus, we have for all $j\in [d]$
	\begin{align}
		\frac{2}{\tau}|(c_t)_j-(c_t^{(i)})_j| \geq |\sign(c_t)_j-\sign(c_t^{(i)})_j|.
	\end{align}
	Then we can obtain
	\begin{align}
		\frac{2}{\tau} \|c_t-c_t^{(i)}\|_1 \geq \big\|\sign(c_t)-\sign(c_t^{(i)})\big\|_1.
	\end{align}
	
	Since $\|\cdot\| \leq \|\cdot\|_1 \leq \sqrt{d}\|\cdot\|$, we have
	\begin{align}
		\big\|\sign(c_t)-\sign(c_t^{(i)})\big\| & \leq \big\|\sign(c_t)-\sign(c_t^{(i)})\|_1 \nonumber \\
		&\leq
		\frac{2}{\tau}\|c_t-c_t^{(i)}\|_1 \nonumber  \leq \frac{2\sqrt{d}}{\tau}\|c_t-c_t^{(i)}\|.
	\end{align}
	
\end{proof}

\begin{theorem} (Restatement of Theorem~\ref{th:1})
	Assume the sequence $\{w_t,c_t\}_{t=1}^T$ is generated from Algorithm~\ref{alg:1} on dataset $S=\{\xi_1,\xi_2,\cdots,\xi_N\}$. Under the Assumptions~\ref{ass:s1},~\ref{ass:g},~\ref{ass:v}, let $\eta=O(\frac{1}{\sqrt{d}})$, $\lambda=O(1)$ with $\lambda\in [0,\frac{1}{\eta})$, $\beta_1=O(1)$ with $\beta_1\in [0,1)$,  $\beta_2=O(1)$ with $\beta_2\in [0,1)$, $\sigma=O(1)$ and
	$L=O(1)$, we have
	\begin{align}
		|\E [F(w_T) - F_S(w_T)]| \leq O(\frac{1}{\tau^T N}),
	\end{align}
	where $\tau=\min_{t\geq 1}(\min_{j\in S_t}(|(c_t)_j|))>0$ with $S_t=\{j|\ |(c_t)_j|\neq 0, j=1,2\cdots,d\}$.
\end{theorem}

\begin{proof}
	Implementing Algorithm~\ref{alg:1} on datasets $S$ and $S^{(i)}$ with
	the same random index sequence $\{j_t\}_{t=1}^T$, and let $\{w_t\}_{t=1}^T$ and $\{w_t^{(i)}\}_{t=1}^T$ be generated from Algorithm~\ref{alg:1} with $S$ and $S^{(i)}$, respectively.
	
	From Algorithm~\ref{alg:1}, since $w_{t} = w_{t-1} - \eta (\sign(c_t)+\lambda w_{t-1})$ and $w_{t}^{(i)} = w_{t-1}^{(i)} - \eta (\sign(c_t^{(i)}) + \lambda w_{t-1}^{(i)})$, we have
	\begin{align}
		w_{t} - w_{t}^{(i)}  = (1-\eta\lambda)(w_{t-1} -w_{t-1}^{(i)}) - \eta \big(\sign(c_t) -
		\sign(c_t^{(i)})\big).
	\end{align}
	Then we have
	\begin{align}
		\|w_{t} - w_{t}^{(i)}\| & =\|(1-\eta\lambda)(w_{t-1} -w_{t-1}^{(i)}) - \eta \big(\sign(c_t) -
		\sign(c_t^{(i)})\big)\| \nonumber \\
		& \leq (1-\eta\lambda)\|w_{t-1} -w_{t-1}^{(i)}\| + \eta \|\sign(c_t) -
		\sign(c_t^{(i)})\| \nonumber \\
		& \leq (1-\eta\lambda)\|w_{t-1} -w_{t-1}^{(i)}\|  + \frac{2\eta\sqrt{d}}{\tau}\|c_t-c_t^{(i)}\|,
	\end{align}
	where $\lambda \in [0,\frac{1}{\eta})$.
	
	If $j_1\neq i$ with probability $1-\frac{1}{N}$, since $c_1=(1-\beta_1)\nabla f(w_0;\xi_{j_1})$, $c_1^{(i)}=(1-\beta_1)\nabla f(w_0^{(i)};\xi_{j_1})$, $m_1=(1-\beta_2)\nabla f(w_0;\xi_{j_1})$, $m_1^{(i)}=(1-\beta_2)\nabla f(w_0^{(i)};\xi_{j_1})$ and $w_0=w_0^{(i)}$, we have $c_1=c_1^{(i)}$ and $m_1=m_1^{(i)}$.
	
	If $j_1= i$ with probability $\frac{1}{N}$, we have
	\begin{align}
		& \E \|c_1 - c_1^{(i)}\| \nonumber \\
		& = \frac{1}{N}\E\|(1-\beta_1)\nabla f(w_0;\xi_i)-(1-\beta_1)\nabla f(w_0^{(i)};\tilde{\xi}_i)\|
		\nonumber \\
		& =\frac{(1-\beta_1)}{N}\E\|\nabla f(w_0;\xi_i)-\nabla F(w_0) + \nabla F(w_0) - \nabla F(w_0^{(i)})+\nabla F(w_0^{(i)})-\nabla f(w_0^{(i)};\tilde{\xi}_i)\| \nonumber \\
		& \leq \frac{2(1-\beta_1)\sigma}{N} + \frac{(1-\beta_1)}{N}\E\| \nabla F(w_0) - \nabla F(w_0^{(i)})\| \nonumber \\
		& = \frac{2(1-\beta_1)\sigma}{N},
	\end{align}
	where the last equality is due to $w_0=w_0^{(i)}$.
	At the same time, we have
	\begin{align}
		& \E \|m_1 - m_1^{(i)}\| \nonumber \\
		& = \frac{1}{N}\E\|(1-\beta_2)\nabla f(w_0;\xi_i)-(1-\beta_2)\nabla f(w_0^{(i)};\tilde{\xi}_i)\|
		\nonumber \\
		& =\frac{(1-\beta_2)}{N}\E\|\nabla f(w_0;\xi_i)-\nabla F(w_0) + \nabla F(w_0) - \nabla F(w_0^{(i)})+\nabla F(w_0^{(i)})-\nabla f(w_0^{(i)};\tilde{\xi}_i)\| \nonumber \\
		& \leq \frac{2(1-\beta_2)\sigma}{N} + \frac{(1-\beta_2)}{N}\E\| \nabla F(w_0) - \nabla F(w_0^{(i)})\| \nonumber \\
		& = \frac{2(1-\beta_2)\sigma}{N}=\frac{\psi_1}{N},
	\end{align}
	where the last inequality holds by $\psi_1=2(1-\beta_2)\sigma$.
	
	Since $w_{0} =w_{0}^{(i)}$, we have
	\begin{align} \label{eq:w1}
		\E \|w_1 - w_1^{(i)}\| & \leq (1-\eta\lambda)\E\|w_{0} -w_{0}^{(i)}\|  +  \frac{2\eta\sqrt{d}}{\tau}\|c_1-c_1^{(i)}\| \nonumber \\
		& \leq  \frac{2\eta\sqrt{d}}{\tau}\frac{2(1-\beta_1)\sigma}{N} = \frac{\phi_1}{\tau N},
	\end{align}
	where $\phi_1=4\eta\sqrt{d}(1-\beta_1)\sigma$.
	
	Let $\eta=O(\frac{1}{\sqrt{d}})$, $\beta_1=O(1)$ with $\beta_1\in [0,1)$ and $\sigma=O(1)$, we have
	$\phi_1=4\eta\sqrt{d}(1-\beta_1)\sigma=O(1)$ and $\psi_1=2(1-\beta_2)\sigma=O(1)$.
	Then we have
	\begin{align}
		\E \|w_1 - w_1^{(i)}\| \leq O(\frac{1}{\tau N}).
	\end{align}
	
	If $j_2\neq i$ with probability $1-\frac{1}{N}$, since $m_2=\beta_2m_1 + (1-\beta_2)\nabla f(w_1;\xi_{j_2})$ and $m_2^{(i)}=\beta_2 m_1^{(i)} + (1-\beta_2)\nabla f(w_1^{(i)};\xi_{j_2})$, we have
	\begin{align}
		\E\|m_2 - m_2^{(i)}  \|
		& = (1-\frac{1}{N})\E \|\beta_2(m_1-m_1^{(i)} )  + (1-\beta_2)(\nabla f(w_1;\xi_{j_2}) -\nabla f(w_1^{(i)};z_{j_2}) )\| \nonumber \\
		& \leq (1-\frac{1}{N})\big(\beta_2 \E\|m_1-m_1^{(i)}\| + (1-\beta_2) \E \|\nabla f(w_1;\xi_{j_2}) -\nabla f(w_1^{(i)};\xi_{j_2})\| \big) \nonumber \\
		& \mathop{\leq}^{(i)} (1-\frac{1}{N})\beta_2\frac{\psi_1}{N} + (1-\frac{1}{N})(1-\beta_2) L \E \|w_1-w_1^{(i)}\| \nonumber \\
		& \leq (1-\frac{1}{N})\beta_2\frac{\psi_1}{N} + (1-\frac{1}{N})(1-\beta_2) L\frac{\phi_1}{\tau N},
	\end{align}
	where the inequality $(i)$ holds by Assumption~\ref{ass:s1}, and the last inequality holds by the above inequality~(\ref{eq:w1}).
	
	Since $c_2=\beta_1m_1 + (1-\beta_1)\nabla f(w_1;\xi_{j_2})$ and $c_2^{(i)}=\beta_1 m_1^{(i)} + (1-\beta_1)\nabla f(w_1^{(i)};\xi_{j_2})$, we can obtain
	\begin{align}
		\E\|c_2 - c_2^{(i)}  \|
		& = (1-\frac{1}{N})\E \big\|\beta_1(m_1-m_1^{(i)} )  + (1-\beta_1)(\nabla f(w_1;\xi_{j_2}) -\nabla f(w_1^{(i)};z_{j_2}) )^2\big\| \nonumber \\
		& \leq (1-\frac{1}{N})\big(\beta_1 \E\|m_1-m_1^{(i)}\| + (1-\beta_1) \E \|\nabla f(w_1;\xi_{j_2}) -\nabla f(w_1^{(i)};\xi_{j_2})\| \big) \nonumber \\
		& \leq (1-\frac{1}{N})\beta_1\frac{\psi_1}{N} + (1-\frac{1}{N})(1-\beta_1) L \E \|w_1-w_1^{(i)}\| \nonumber \\
		& \leq (1-\frac{1}{N})\beta_1\frac{\psi_1}{N} + (1-\frac{1}{N})(1-\beta_1) L \frac{\phi_1}{\tau N}.
	\end{align}
	
	If $j_2= i$ with probability $\frac{1}{N}$, we have
	\begin{align}
		\E\|m_2 - m_2^{(i)} \|
		& = \frac{1}{N}\E \big\| \beta_2(m_1-m_1^{(i)} )  + (1-\beta_2)(\nabla f(w_1;\xi_{i}) -\nabla f(w_1^{(i)};\tilde{\xi}_{i}) )\big\| \nonumber \\
		& \leq \frac{1}{N}\beta_2 \E\|m_1-m_1^{(i)}\| + \frac{1}{N}(1-\beta_2) \E \|\nabla f(w_1;\xi_{i}) -\nabla f(w_1^{(i)};\tilde{\xi}_{i})\| \nonumber \\
		& \leq \frac{1}{N}\beta_2\frac{\psi_1}{N} + \frac{1-\beta_2}{N} \E \big\|\nabla f(w_1;\xi_{i}) - \nabla F(w_1) + \nabla F(w_1) - \nabla F(w_1^{(i)}) \nonumber \\
		& \quad + \nabla F(w_1^{(i)})-\nabla f(w_1^{(i)};\tilde{\xi}_{i})\big\| \nonumber \\
		& \leq  \frac{1}{N}\beta_2\frac{\psi_1}{N} + \frac{1-\beta_2}{N} \big( \E \|\nabla f(w_1;\xi_{i}) - \nabla F(w_1)\| + \E\|\nabla F(w_1) - \nabla F(w_1^{(i)})\| \nonumber \\
		&\quad + \E\|\nabla F(w_1^{(i)})-\nabla f(w_1^{(i)};\tilde{\xi}_{i})\| \big) \nonumber \\
		& \mathop{\leq}^{(i)}  \frac{1}{N}\frac{\beta_2\psi_1}{N}  + \frac{2(1-\beta_2)\sigma}{N} + \frac{(1-\beta_2)L}{N} \E \|w_1-w_1^{(i)}\| \nonumber \\
		& \leq \frac{1}{N}\frac{\beta_2\psi_1}{N}  + \frac{2(1-\beta_2)\sigma}{N} + \frac{(1-\beta_2)L}{N} \frac{\phi_1}{\tau N},
	\end{align}
	where the above inequality~$(i)$ holds by Assumptions~\ref{ass:s1} and~\ref{ass:v}.
	At the same time, we also have
	\begin{align}
		\E\|c_2 - c_2^{(i)}  \|
		& = \frac{1}{N}\E \big\|\beta_1(m_1-m_1^{(i)} )  + (1-\beta_1)(\nabla f(w_1;\xi_{i}) -\nabla f(w_1^{(i)};\tilde{\xi}_{i}) \big\| \nonumber \\
		& \leq \frac{1}{N}\big(\beta_1 \E\|m_1-m_1^{(i)}\| + (1-\beta_1) \E \|\nabla f(w_1;\xi_{i}) -\nabla f(w_1^{(i)};\tilde{\xi}_{i})\| \big) \nonumber \\
		& \leq \frac{1}{N}\beta_1\frac{\psi_1}{N} + \frac{1}{N}(1-\beta_1) \E \|\nabla f(w_1;\xi_{i}) - \nabla F(w_1) +\nabla F(w_1) -\nabla F(w_1^{(i)}) \nonumber \\
		& \quad + \nabla F(w_1^{(i)}) -\nabla f(w_1^{(i)};\tilde{\xi}_{i})\| \nonumber \\
		& \leq \frac{1}{N}\beta_1\frac{\psi_1}{N} + \frac{1}{N}(1-\beta_1) \big(2\sigma+ L \E\|w_1-w_1^{(i)}\| \big) \nonumber \\
		& \leq  \frac{1}{N}\beta_1\frac{\psi_1}{N} + \frac{2\sigma}{N}(1-\beta_1) + \frac{L(1-\beta_1)}{N}\frac{\phi_1}{\tau N}.
	\end{align}
	
	Thus, we can obtain
	\begin{align}
		\E\|m_2 - m_2^{(i)}  \|
		& \leq (1-\frac{1}{N})\beta_2\frac{\psi_1}{N} + (1-\frac{1}{N})(1-\beta_2) L\frac{\phi_1}{\tau N} \nonumber \\
		& \quad + \frac{1}{N}\frac{\beta_2\psi_1}{N}  + \frac{2(1-\beta_2)\sigma}{N} + \frac{(1-\beta_2)L}{ N} \frac{\phi_1}{\tau N} \nonumber \\
		& = \beta_2\frac{\psi_1}{N}  + \frac{2(1-\beta_2)\sigma}{N} +  (1-\beta_2)L\frac{\phi_1}{\tau N},
	\end{align}
	and
	\begin{align}
		\E\|c_2 - c_2^{(i)}  \|
		& \leq (1-\frac{1}{N})\beta_1\frac{\psi_1}{N} + (1-\frac{1}{N})(1-\beta_1) L \frac{\phi_1}{\tau N} \nonumber \\
		& \quad + \frac{1}{N}\beta_1\frac{\psi_1}{N} + \frac{2\sigma}{N}(1-\beta_1) + \frac{L(1-\beta_1)}{N}\frac{\phi_1}{\tau N} \nonumber \\
		& \leq \beta_1\frac{\psi_1}{N} + (1-\beta_1)\frac{2\sigma}{N} + (1-\beta_1)L\frac{\phi_1}{\tau N}.
	\end{align}
	Then we have
	\begin{align}
		\E \|w_2 - w_2^{(i)}\| & \leq (1-\eta\lambda)\E\|w_{1} -w_{1}^{(i)}\|  +  \frac{2\eta\sqrt{d}}{\tau}\|c_2-c_2^{(i)}\| \nonumber \\
		& \leq  (1-\eta\lambda)\frac{\phi_1}{\tau N} +  \frac{2\eta\sqrt{d}}{\tau}\big(\beta_1\frac{\psi_1}{N} + (1-\beta_1)\frac{2\sigma}{N} + (1-\beta_1)L\frac{\phi_1}{\tau N}\big).
	\end{align}
	
	Let $\psi_2 = \beta_2\psi_1  + 2(1-\beta_2)\sigma +  (1-\beta_2)L\phi_1$ and
	$\phi_2= (1-\eta\lambda)\phi_1 +  2\eta\sqrt{d}\big(\beta_1\psi_1 + 2(1-\beta_1)\sigma + (1-\beta_1)L\phi_1\big)$.
	
	Let $\eta=O(\frac{1}{\sqrt{d}})$, $\lambda=O(1)$ with $\lambda\in [0,\frac{1}{\eta})$, $\beta_1=O(1)$ with $\beta_1\in [0,1)$,  $\beta_2=O(1)$ with $\beta_2\in [0,1)$, $\sigma=O(1)$ and
	$L=O(1)$. Since $\psi_1=O(1)$ and $\phi_1=O(1)$, we have
	\begin{align}
		& \psi_2 = \beta_2\psi_1  + 2(1-\beta_2)\sigma +  (1-\beta_2)L\frac{\phi_1}{\tau} =O(\frac{1}{\tau}) \nonumber \\
		& \phi_2= (1-\eta\lambda)\phi_1 +  2\eta\sqrt{d}\big(\beta_1\psi_1 + 2(1-\beta_1)\sigma + (1-\beta_1)L\frac{\phi_1}{\tau}\big) =O(\frac{1}{\tau}).
	\end{align}
	Thus, we have
	\begin{align}
		\E\|m_2 - m_2^{(i)}  \|
		& \leq  \beta_2\frac{\psi_1}{N}  + \frac{2(1-\beta_2)\sigma}{N} +  (1-\beta_2)L\frac{\phi_1}{\tau N}=\frac{\psi_2}{ N}=O(\frac{1}{\tau N}),
	\end{align}
	and
	\begin{align}
		\E \|w_2 - w_2^{(i)}\|
		& \leq  (1-\eta\lambda)\frac{\phi_1}{\tau N} +  \frac{2\eta\sqrt{d}}{\tau}\big(\beta_1\frac{\psi_1}{N} + (1-\beta_1)\frac{2\sigma}{N} + (1-\beta_1)L\frac{\phi_1}{\tau N}\big) \nonumber \\
		& =\frac{\phi_2}{\tau N}=O(\frac{1}{\tau^2 N}).
	\end{align}
	
	Based on mathematical induction, we assume $\E \|w_t - w_t^{(i)}\| \leq \frac{\phi_t}{\tau N}$ with
	$\phi_t=O(\frac{1}{\tau^{t-1}})$, and $\E\|m_t-m_t^{(i)}\|\leq \frac{\psi_t}{N}$ with $\psi_t=O(\frac{1}{\tau^{t-1}})$.
	
	If $j_{t+1}\neq i$ with probability $1-\frac{1}{N}$, since $m_{t+1}=\beta_2m_t + (1-\beta_2)\nabla f(w_t;\xi_{j_{t+1}})$ and $m_{t+1}^{(i)}=\beta_2 m_t^{(i)} + (1-\beta_2)\nabla f(w_{t}^{(i)};\xi_{j_{t+1}})$, we have
	\begin{align}
		\E\|m_{t+1} - m_{t+1}^{(i)}  \|
		& = (1-\frac{1}{N})\E \|\beta_2(m_{t}-m_{t}^{(i)} )  + (1-\beta_2)(\nabla f(w_t;\xi_{j_{t+1}}) -\nabla f(w_t^{(i)};\xi_{j_{t+1}}) )\| \nonumber \\
		& \leq (1-\frac{1}{N})\big(\beta_2 \E\|m_t-m_t^{(i)}\| + (1-\beta_2) \E \|\nabla f(w_t;\xi_{j_{t+1}}) -\nabla f(w_t^{(i)};\xi_{j_{t+1}})\| \big) \nonumber \\
		& \leq (1-\frac{1}{N})\beta_2\frac{\psi_t}{N} + (1-\frac{1}{N})(1-\beta_2) L \E \|w_t-w_t^{(i)}\| \nonumber \\
		& \leq (1-\frac{1}{N})\beta_2\frac{\psi_t}{N} + (1-\frac{1}{N})(1-\beta_2) L\frac{\phi_t}{\tau N},
	\end{align}
	where the second last inequality holds by Assumption~\ref{ass:s1}.
	Since $c_{t+1}=\beta_1m_t + (1-\beta_1)\nabla f(w_t;\xi_{j_{t+1}})$ and $c_{t+1}^{(i)}=\beta_1 m_t^{(i)} + (1-\beta_1)\nabla f(w_t^{(i)};\xi_{j_{t+1}})$, then we have
	\begin{align}
		\E\|c_{t+1} - c_{t+1}^{(i)}  \|
		& = (1-\frac{1}{N})\E \big\|\beta_1(m_t-m_t^{(i)} )  + (1-\beta_1)(\nabla f(w_t;\xi_{j_{t+1}}) -\nabla f(w_t^{(i)};\xi_{j_{t+1}}) \big\| \nonumber \\
		& \leq (1-\frac{1}{N})\big(\beta_1 \E\|m_t-m_t^{(i)}\| + (1-\beta_1) \E \|\nabla f(w_t;\xi_{j_{t+1}}) -\nabla f(w_t^{(i)};\xi_{j_{t+1}})\| \big) \nonumber \\
		& \leq (1-\frac{1}{N})\beta_1\frac{\psi_t}{N} + (1-\frac{1}{N})(1-\beta_1) \E \|\nabla f(w_t;\xi_{j_{t+1}})-\nabla f(w_t^{(i)};\xi_{j_{t+1}})\| \nonumber \\
		& \leq (1-\frac{1}{N})\beta_1\frac{\psi_t}{N} + (1-\frac{1}{N})(1-\beta_1)L\E \|w_t-w_t^{(i)}\| \nonumber \\
		& \leq (1-\frac{1}{N})\beta_1\frac{\psi_t}{N} + (1-\frac{1}{N})(1-\beta_1) L\frac{\phi_t}{\tau N}.
	\end{align}
	
	If $j_{t+1}= i$ with probability $\frac{1}{N}$, we can obtain
	\begin{align}
		\E\|m_{t+1} - m_{t+1}^{(i)} \|
		& = \frac{1}{N}\E \big\| \beta_2(m_t-m_t^{(i)} )  + (1-\beta_2)(\nabla f(w_t;\xi_{i}) -\nabla f(w_t^{(i)};\tilde{\xi}_{i}) )\big\| \nonumber \\
		& \leq \frac{1}{N}\beta_2 \E\|m_t-m_t^{(i)}\| + \frac{1}{N}(1-\beta_2) \E \|\nabla f(w_t;\xi_{i}) -\nabla f(w_t^{(i)};\tilde{\xi}_{i})\| \nonumber \\
		& \leq \frac{1}{N}\beta_2\frac{\psi_t}{N} + \frac{1-\beta_2}{N} \E \big\|\nabla f(w_t;\xi_{i}) - \nabla F(w_t) + \nabla F(w_t) - \nabla F(w_t^{(i)}) \nonumber \\
		& \quad + \nabla F(w_t^{(i)})-\nabla f(w_t^{(i)};\tilde{\xi}_{i})\big\| \nonumber \\
		& \leq  \frac{1}{N}\beta_2\frac{\psi_t}{N} + \frac{1-\beta_2}{N} \big( \E \|\nabla f(w_t;\xi_{i}) - \nabla F(w_t)\| + \E\|\nabla F(w_t) - \nabla F(w_t^{(i)})\| \nonumber \\
		&\quad + \E\|\nabla F(w_t^{(i)})-\nabla f(w_t^{(i)};\tilde{\xi}_{i})\| \big) \nonumber \\
		& \leq  \frac{1}{N}\beta_2\frac{\psi_t}{N}  + \frac{2(1-\beta_2)\sigma}{N} + \frac{(1-\beta_2)L}{N} \E \|w_t-w_t^{(i)}\| \nonumber \\
		& \leq \frac{1}{N}\beta_1\frac{\psi_t}{N}  + \frac{2(1-\beta_2)\sigma}{N} + \frac{(1-\beta_2)L}{N} \frac{\phi_t}{\tau N},
	\end{align}
	and
	\begin{align}
		\E\|c_{t+1} - c_{t+1}^{(i)}  \|
		& = \frac{1}{N}\E \big\|\beta_1(m_t-m_t^{(i)} )  + (1-\beta_1)(\nabla f(w_t;\xi_{i}) -\nabla f(w_t^{(i)};\tilde{\xi}_{i}) \big\| \nonumber \\
		& \leq \frac{1}{N}\big(\beta_1 \E\|m_t-m_t^{(i)}\| + (1-\beta_1) \E \|\nabla f(w_t;\xi_{i}) -\nabla f(\theta_t^{(i)};\tilde{z}_{i})\| \big) \nonumber \\
		& \leq \frac{1}{N}\beta_1\frac{\psi_t}{N} + \frac{1}{N}(1-\beta_1) \E \|\nabla f(w_t;\xi_{i})-\nabla f(w_t^{(i)};\tilde{\xi}_{i})\| \nonumber \\
		& \leq  \frac{1}{N}\beta_1\frac{\psi_t}{N} + \frac{1}{N}(1-\beta_1) \E \|\nabla f(w_t;\xi_{i})-\nabla F(w_t) + \nabla F(w_t) - \nabla F(w_t^{(i)}) \nonumber \\
		& \qquad + \nabla F(w_t^{(i)})-\nabla f(w_t^{(i)};\tilde{\xi}_{i})\| \nonumber \\
		& \leq \frac{1}{N}\beta_1\frac{\psi_t}{N} + \frac{1}{N}(1-\beta_1) \big(2\sigma+ L \E\|w_t-w_t^{(i)}\| \big) \nonumber \\
		& \leq  \frac{1}{N}\beta_1\frac{\psi_t}{N} + \frac{1}{N}(1-\beta_1) 2\sigma + \frac{1}{N}(1-\beta_1) L \frac{\phi_t}{\tau N}.
	\end{align}
	
	Let $\psi_{t+1}= \beta_2\psi_t + 2(1-\beta_2)\sigma + (1-\beta_2)L\frac{\phi_t}{\tau}$, we have
	\begin{align} \label{eq:m}
		\E\|m_{t+1} - m_{t+1}^{(i)}  \|
		& \leq (1-\frac{1}{N})\beta_2\frac{\psi_t}{N} + (1-\frac{1}{N})(1-\beta_2) L\frac{\phi_t}{\tau N} \nonumber \\
		&\quad + \frac{1}{N}\beta_1\frac{\psi_t}{N}  + \frac{2(1-\beta_2)\sigma}{N} + \frac{(1-\beta_2)L}{N} \frac{\phi_t}{\tau N} \nonumber \\
		& = \frac{\beta_2\psi_t}{N}  + \frac{2(1-\beta_2)\sigma}{N} + \frac{(1-\beta_2)L\phi_t}{\tau N} \nonumber \\
		& = \frac{\psi_{t+1}}{N}.
	\end{align}
	At the same time, we also have
	\begin{align} \label{eq:v}
		\E\|c_{t+1} - c_{t+1}^{(i)}  \|
		& \leq (1-\frac{1}{N})\beta_1\frac{\psi_t}{N} + (1-\frac{1}{N})(1-\beta_1) L\frac{\phi_t}{\tau N} \nonumber \\
		&\quad +\frac{1}{N}\beta_1\frac{\psi_t}{N} + \frac{1}{N}(1-\beta_1) 2\sigma + \frac{1}{N}(1-\beta_1) L \frac{\phi_t}{\tau N} \nonumber \\
		& = \frac{\beta_1 \psi_t}{N} +\frac{ 2(1-\beta_1)\sigma}{N} + \frac{(1-\beta_1)L\phi_t}{\tau N}.
	\end{align}
	Let $\phi_{t+1} = (1-\lambda\eta)\phi_t + 2\eta\sqrt{d}\big( \beta_1 \psi_t + 2(1-\beta_1)\sigma + (1-\beta_1)L\frac{\phi_t}{\tau}\big) $, we have
	\begin{align} \label{eq:w2}
		\E \|w_{t+1} - w_{t+1}^{(i)}\| & \leq (1-\lambda\eta)\E \|w_{t} -w_{t}^{(i)}\| + \frac{2\eta\sqrt{d}}{\tau}\E \|c_{t+1}-c_{t+1}^{(i)}\|  \nonumber \\
		& \leq  (1-\lambda\eta)\frac{\phi_t}{\tau N} + \frac{2\eta\sqrt{d}}{\tau}\big(\frac{\beta_1 \psi_t}{N} +\frac{ 2(1-\beta_1)\sigma}{N} + \frac{(1-\beta_1)L\phi_t}{\tau N} \big) \nonumber \\
		& = \frac{\phi_{t+1}}{\tau N}.
	\end{align}
	Let $\eta=O(\frac{1}{\sqrt{d}})$, $\lambda=O(1)$ with $\lambda\in [0, \frac{1}{\eta})$, $\beta_1=O(1)$ with $\beta_1\in [0,1)$,  $\beta_2=O(1)$ with $\beta_2\in [0,1)$, $\sigma=O(1)$ and
	$L=O(1)$. Since $\psi_t=O(\frac{1}{\tau^{t-1}})$ and $\phi_t=O(\frac{1}{\tau^{t-1}})$, we have
	\begin{align} \label{eq:v}
		& \psi_{t+1}= \beta_2\psi_t + 2(1-\beta_2)\sigma + (1-\beta_2)L\frac{\phi_t}{\tau}=O(\frac{1}{\tau^{t}}), \nonumber \\
		& \phi_{t+1} = (1-\lambda\eta)\phi_t + 2\eta\sqrt{d}\big( \beta_1 \psi_t + 2(1-\beta_1)\sigma + (1-\beta_1)L\frac{\phi_t}{\tau}\big) =O(\frac{1}{\tau^{t}}).
	\end{align}
	Then we have
	\begin{align}
		\E \|w_{t+1} - w_{t+1}^{(i)}\|  \leq  \frac{\phi_{t+1}}{\tau N} =O(\frac{1}{\tau^{t+1} N}).
	\end{align}
	
	By using mathematical induction, then we have
	\begin{align} \label{eq:59}
		\E \|w_{T} - w_{T}^{(i)}\| \leq O(\frac{1}{\tau^T N}).
	\end{align}
	
	By using Assumption~\ref{ass:g}, i.e., the condition of $G$-Lipschitz $f(w;\xi)$ (i.e.,), we have
	for any $\xi\sim \mathcal{D}$
	\begin{align}  \label{eq:60}
		\E |f(w_T;\xi)-f(w_T^{(i)};\xi)| \leq G \E \|w_{T} - w_{T}^{(i)}\| \leq O(\frac{1}{\tau^T N}),
	\end{align}
	where the last inequality holds by the above inequality~(\ref{eq:59}) and $G=O(1)$.
	
	By using the lemma~\ref{lem:gs}, i.e., the uniform stability bound~\citep{shalev2010learnability,hardt2016train}, and taking expectations over $S$, $S^{(i)}$ and the algorithm's randomness on the above inequality~(\ref{eq:60}), we can obtain
	\begin{align}
		|\E [F(w_T) - F_S(w_T)]| \leq O(\frac{1}{\tau^T N}).
	\end{align}

\end{proof}

\section{Generalization Analysis of our CLion Optimizer}
\label{ga:clion}

\begin{theorem} (Restatement of Theorem~\ref{th:2})
	Assume the sequence $\{w_t\}_{t=1}^T$ is generated
	from Algorithm~\ref{alg:2} on dataset $S=\{\xi_1,\xi_2,\cdots,\xi_N\}$. Under the Assumptions~\ref{ass:s1},~\ref{ass:g},~\ref{ass:v}, without loss of generality, let $\nu\geq 1$, $\lambda=O(1)$ with $0<\lambda\leq \frac{1}{\eta}$, $\beta_1=O(1)$ with $\beta_1\in (0,1)$, $\beta_2=O(1)$ with $\beta_2\in (0,1)$, $\sigma=O(1)$, $L=O(1)$ and $G=O(1)$.
	When the iteration number is small (i.e., $T=O(1)$) set
	$\eta=\frac{1}{\sqrt{d}}$, otherwise set $\eta=\frac{1}{\sqrt{d}T}$, we have
	\begin{align}
		|\E [F(w_T) - F_S(w_T)]| \leq O(\frac{1}{N}).
	\end{align}
\end{theorem}

\begin{proof}
	Let $c_t$ be generated from Algorithm~\ref{alg:2} for any $t\geq 1$. Meanwhile, we combine the line 8 with
	the line 10 of Algorithm~\ref{alg:2} by the following formation,
	\begin{align}
		w_{t} = w_{t-1} - \eta \big( h(c_t)+\lambda w_{t-1} \big),
	\end{align}
	where $h(c_t) =\sign(c_t)$ when $\min_{ j\in S_t} |(c_t)_j| \geq \nu$ with $S_t=\{j|(c_t)_j\neq 0, j=1,\cdots,d\}$, otherwise $h(c_t)=c_t$.
	
	When $\min_{ j\in S_t} |(c_t)_j| \geq \nu$ with $S_t=\{j|(c_t)_j\neq 0, j=1,\cdots,d\}$, based on the above proof of Lemma~\ref{lem:1}, we can obtain
	\begin{align}
		\big\|h(c_t)-h(c_t^{(i)})\big\|=\big\|\sign(c_t)-\sign(c_t^{(i)})\big\| \leq
		\frac{2\sqrt{d}}{\nu}\|c_t-c_t^{(i)}\|.
	\end{align}
	Without loss of generality, let $\nu\geq 1$, we have
	\begin{align}
		\big\|h(c_t)-h(c_t^{(i)})\big\|=\big\|\sign(c_t)-\sign(c_t^{(i)})\big\| \leq
		\frac{2\sqrt{d}}{\nu}\|c_t-c_t^{(i)}\| \leq 2\sqrt{d}\|c_t-c_t^{(i)}\|.
	\end{align}
	When $\min_{ j\in S_t} |(c_t)_j| < \nu$, we have
	\begin{align}
		\big\|h(c_t)-h(c_t^{(i)})\big\|=\big\|c_t-c_t^{(i)}\big\| \leq 2\sqrt{d}\|c_t-c_t^{(i)}\|.
	\end{align}
	Thus we have for all $t\geq 0$
	\begin{align}
		\big\|h(c_t)-h(c_t^{(i)})\big\| \leq 2\sqrt{d}\|c_t-c_t^{(i)}\|.
	\end{align}
	
	Implementing Algorithm~\ref{alg:2} on datasets $S$ and $S^{(i)}$ with
	the same random index sequence $\{j_t\}_{t=1}^T$, and let $\{w_t\}_{t=1}^T$ and $\{w_t^{(i)}\}_{t=1}^T$ be generated from Algorithm~\ref{alg:2} with $S$ and $S^{(i)}$, respectively.
	
	From Algorithm~\ref{alg:2}, $w_{t} = w_{t-1} - \eta \big( h(c_t)+\lambda w_{t-1} \big)$ and $w_{t}^{(i)} = w_{t-1}^{(i)} - \eta \big( h(c_t^{(i)}) + \lambda w_{t-1}^{(i)}\big)$, we have
	\begin{align}
		w_{t} - w_{t}^{(i)}  = (1-\eta\lambda)(w_{t-1} -w_{t-1}^{(i)}) - \eta \big(h(c_t)-h(c_t^{(i)})\big).
	\end{align}
	Then we have
	\begin{align} \label{eq:w3}
		\|w_{t} - w_{t}^{(i)}\| & =\|(1-\eta\lambda)(w_{t-1} -w_{t-1}^{(i)}) - \eta \big(h(c_t)-h(c_t^{(i)})\big)\| \nonumber \\
		& \leq (1-\eta\lambda)\|w_{t-1} -w_{t-1}^{(i)}\| + \eta \|h(c_t)-h(c_t^{(i)})\| \nonumber \\
		& \leq (1-\eta\lambda)\|w_{t-1} -w_{t-1}^{(i)}\|  + 2\eta\sqrt{d}\|c_t-c_t^{(i)}\|,
	\end{align}
	where $0<\lambda \leq \frac{1}{\eta}$.
	
	Following the above proof of Theorem~\ref{th:1}, by using the above inequalities~(\ref{eq:w2}) and~(\ref{eq:w3}),
	we have
	\begin{align} \label{eq:w4}
		\E \|w_{t+1} - w_{t+1}^{(i)}\| & \leq (1-\lambda\eta)\E \|w_{t} -w_{t}^{(i)}\| + 2\eta\sqrt{d}\E \|c_{t+1}-c_{t+1}^{(i)}\|  \nonumber \\
		& \leq  (1-\lambda\eta)\frac{\phi_t}{ N} + 2\eta\sqrt{d}\big(\frac{\beta_1 \psi_t}{N} +\frac{ 2(1-\beta_1)\sigma}{N} + \frac{(1-\beta_1)L\phi_t}{ N} \big) = \frac{\phi_{t+1}}{N},
	\end{align}
	where $\phi_{t+1} = (1-\lambda\eta)\phi_t + 2\eta\sqrt{d}\big( \beta_1 \psi_t + 2(1-\beta_1)\sigma + (1-\beta_1)L\phi_t\big) $. Similarly, by using the above inequalities~(\ref{eq:m}) and~(\ref{eq:w3}),
	we have
	\begin{align} \label{eq:m1}
		\E\|m_{t+1} - m_{t+1}^{(i)}  \|
		\leq \frac{\beta_2\psi_t}{N}  + \frac{2(1-\beta_2)\sigma}{N} + \frac{(1-\beta_2)L\phi_t}{N}  = \frac{\psi_{t+1}}{N},
	\end{align}
	where $\psi_{t+1}= \beta_2\psi_t + 2(1-\beta_2)\sigma + (1-\beta_2)L\phi_t$.

	\textbf{When} the iteration number is small (i.e., $T=O(1)$),
	following the above proof of Theorem~\ref{th:1}, let $\eta=\frac{1}{\sqrt{d}}$, $\lambda=O(1)$ with $0<\lambda\leq \frac{1}{\eta}$, $\beta_1=O(1)$ with $\beta_1\in (0,1)$, $\beta_2=O(1)$ with $\beta_2\in (0,1)$, $\sigma=O(1)$ and $L=O(1)$.
	Assume $\E \|w_t - w_t^{(i)}\| \leq \frac{\phi_t}{N}$ with
	$\phi_t=O(1)$, and $\E\|m_t-m_t^{(i)}\|\leq \frac{\psi_t}{N}$ with $\psi_t=O(1)$, we can obtain $\phi_{t+1} =O(1)$ and $\psi_{t+1}=O(1)$.
	By using the above inequality~(\ref{eq:w4}), then we have
	\begin{align}
		\E \|w_{t+1} - w_{t+1}^{(i)}\| \leq \frac{\phi_{t+1}}{N} =O(\frac{1}{N}).
	\end{align}
	Based on the mathematical induction, we have
	\begin{align}
		\E \|w_{T} - w_{T}^{(i)}\| &  \leq  \frac{\phi_T}{N} = O(\frac{1}{N}).
	\end{align}

	\textbf{When} the totally iteration number $T$ is large, we consider the iteration number $t\geq1$
	in the following generalization analysis. By using mathematical induction, due to recursion of the above inequality~(\ref{eq:w4}), we assume $\E \|w_t - w_t^{(i)}\| \leq \frac{\phi_t}{N}$ with
	$\phi_t=O(t)$, and $\E\|m_t-m_t^{(i)}\|\leq \frac{\psi_t}{N}$ with $\psi_t=O(t)$.
	Following the above proof of Theorem~\ref{th:1}, let $\eta=O(\frac{1}{\sqrt{d}})$, $\lambda=O(1)$ with $0<\lambda\leq \frac{1}{\eta}$, $\beta_1=O(1)$ with $\beta_1\in (0,1)$, $\beta_2=O(1)$ with $\beta_2\in (0,1)$, $\sigma=O(1)$ and
	$L=O(1)$, we can obtain $\phi_{t+1} =O(t+1)$ and $\psi_{t+1}=O(t+1)$.
	By using the above inequality~(\ref{eq:w4}), then we have
	\begin{align}
		\E \|w_{t+1} - w_{t+1}^{(i)}\| \leq \frac{\phi_{t+1}}{N} =O(\frac{t+1}{N}).
	\end{align}
	Based on the mathematical induction, we have
	\begin{align}
		\E \|w_{T} - w_{T}^{(i)}\| &  \leq  \frac{\phi_T}{N} = O(\frac{T}{N}).
	\end{align}
	
	Similarly, following the above proof of Theorem~\ref{th:1}, let $\eta=O( \textcolor{red}{\frac{1}{T\sqrt{d}}})$, $\lambda=O(1)$ with $0<\lambda\leq \frac{1}{\eta}$, $\beta_1=O(1)$ with $\beta_1\in (0,1)$, $\beta_2=O(1)$ with $\beta_2\in (0,1)$, $\sigma=O(1)$ and
	$L=O(1)$, we can obtain
	\begin{align} \label{eq:76}
		\E \|w_{T} - w_{T}^{(i)}\| &  \leq  \frac{\phi_T}{N} = O(\frac{1}{N}).
	\end{align}
	
	By using Assumption~\ref{ass:g}, i.e., the condition of $G$-Lipschitz $f(w;\xi)$ (i.e.,), we have
	for any $\xi\sim \mathcal{D}$
	\begin{align}  \label{eq:77}
		\E |f(w_T;\xi)-f(w_T^{(i)};\xi)| \leq G \E \|w_{T} - w_{T}^{(i)}\| \leq O(\frac{1}{N}),
	\end{align}
	where the last inequality holds by the above inequality~(\ref{eq:76}) and $G=O(1)$.
	
	By using the lemma~\ref{lem:gs}, i.e., the uniform stability bound~\citep{shalev2010learnability,hardt2016train}, and taking expectations over $S$, $S^{(i)}$ and the algorithm's randomness on the above inequality~(\ref{eq:77}), we can obtain
	\begin{align}
		|\E [F(w_T) - F_S(w_T)]| \leq O(\frac{1}{N}).
	\end{align}

\end{proof}

\section{Convergence Analysis of our CLion Optimizer}
\label{ca:clion}

\begin{lemma} (Lemma A.7. of \citep{liu2024communication}) \label{lem:s1}
	Let $(X,Y)$ is a joint random variable on $\R^d \times \R^d$. For any constant $a\in (0,+\infty)$, we have
	\begin{align}
		\E [\langle X, \sign(X)-\sign(Y)\rangle  ] \leq 2\alpha \sqrt{d}\E \|X/\alpha -Y\|.
	\end{align}
\end{lemma}

\begin{lemma} \label{lem:s2}
	(\cite{nesterov2018lectures})
	Assume that $f(w)$ is a differentiable convex function and $\mathcal{W}$ is a convex set.
	$w^* \in \mathcal{W}$ is the solution of the
	constrained problem $\min_{w\in \mathcal{W}}f(w)$, if
	$$ \langle \nabla f(w^*), w-w^*\rangle \geq 0, \ \forall w\in \mathcal{W}. $$
\end{lemma}

\begin{lemma} (Restatement of Lemma~\ref{lem:2})
	Assume the sequence $\{w_t\}_{t=1}^T$ is generated from Algorithm~\ref{alg:2}, let $\|w_{0}\|\leq \eta\hat{G}$ and $\lambda \leq \frac{1}{2\eta\hat{G} T^{\alpha}}$ with $\alpha>1$, we have
	\begin{align}
		\|w_t\|\leq (t+1)\eta\hat{G}, \quad  \|w_t - w_{t-1}\|\leq 2\eta \hat{G},
	\end{align}
	where $\hat{G}=\max(G,\sqrt{d})$.
\end{lemma}

\begin{proof}
	When $\min_{ j\in S_t} |(c_t)_j| \geq \nu$ with $S_t=\{j|(c_t)_j\neq 0, j=1,\cdots,d\}$, according to
	the line 8 of Algorithm~\ref{alg:2}, we have
	\begin{align}
		\|w_t\| & = \|w_{t-1}- \eta (\sign(c_t)+\lambda w_{t-1})\| \nonumber \\
		& = \|(1-\eta\lambda)w_{t-1} + \eta\sign(c_t)\| \nonumber \\
		& \leq (1-\eta\lambda)\|w_{t-1}\| + \eta\sqrt{d} \nonumber \\
		& \leq (1-\eta\lambda)^t\|w_{0}\| + t\eta\sqrt{d} \nonumber \\
		& \leq (t+1)\eta\sqrt{d},
	\end{align}
	where $\|w_{0}\|\leq \eta\max(\sqrt{d},G)$.
	
	By using Assumption~\ref{ass:g}, since $\|\nabla f(w;\xi)\|\leq G$ and $\beta_1,\beta_2\in (0,1)$,
	we can easily obtain $\|m_t\|\leq G$ and $\|c_t\|\leq G$.
	
	When $\min_{ j\in S_t} |(c_t)_j| < \nu$, according to
	the line 10 of Algorithm~\ref{alg:2}, we have
	\begin{align}
		\|w_t\| & = \|w_{t-1}- \eta (c_t+\lambda w_{t-1})\| \nonumber \\
		& = \|(1-\eta\lambda)w_{t-1} + \eta c_t\| \nonumber \\
		& \leq (1-\eta\lambda)\|w_{t-1}\| + \eta \|c_t\| \nonumber \\
		& \leq (1-\eta\lambda)\|w_{t-1}\| + \eta G \nonumber \\
		& \leq (1-\eta\lambda)^t\|w_{0}\| + t\eta G \nonumber \\
		& \leq (t+1)\eta G,
	\end{align}
	where $\|w_{0}\|\leq \eta\max(\sqrt{d},G)$.
	
	Let $\hat{G}=\max(G,\sqrt{d})$, we have $\|w_t\|\leq (t+1)\eta\hat{G}$ for all $t\geq 0$.
	When $\min_{ j\in S_t} |(c_t)_j| \geq \nu$, then we have
	\begin{align}
		\|w_t - w_{t-1}\|^2 & = \|\eta (\sign(c_t)+\lambda w_{t-1})\|^2 \nonumber \\
		& \leq 2\eta^2\|\sign(c_t)\|^2 + 2\eta^2\lambda^2\|w_{t-1}\|^2 \nonumber \\
		& \leq 2\eta^2 d + 2\eta^2\lambda^2t^2\eta^2\hat{G}^2 \nonumber \\
		& \leq 4\eta^2 \hat{G}^2,
	\end{align}
	where $\lambda \leq \frac{1}{2\eta \hat{G}T^{\alpha}}$ with $\alpha>1$.
	
	When $\min_{ j\in S_t} |(c_t)_j| < \nu$, then we have
	\begin{align}
		\|w_t - w_{t-1}\|^2 & = \|\eta (c_t+\lambda w_{t-1})\|^2 \nonumber \\
		& \leq 2\eta^2\|c_t\|^2 + 2\eta^2\lambda^2\|w_{t-1}\|^2 \nonumber \\
		& \leq 2\eta^2 G^2 + 2\eta^2\lambda^2t^2\eta^2\hat{G}^2 \nonumber \\
		& \leq 4\eta^2 \hat{G}^2,
	\end{align}
	where $\lambda \leq \frac{1}{2\eta\hat{G} T^{\alpha}}$ with $\alpha>1$.
	
	Thus, we have $\|w_t - w_{t-1}\|^2 \leq 4\eta^2 \hat{G}^2$ for all $t\geq 1$, and it implies that
	$\|w_t - w_{t-1}\|\leq 2\eta \hat{G}$.
\end{proof}

\begin{lemma} (Restatement of Lemma~\ref{lem:3})
	Assume the sequence $\{c_t\}_{t=1}^T$ is generated from Algorithm~\ref{alg:2}, let $\|w_{0}\|\leq \eta\hat{G}$ and $\lambda \leq \frac{1}{2\eta\hat{G} T^{\alpha}}$ with $\alpha>1$, we have
	\begin{align}
		\frac{1}{T} \sum_{t=1}^T\E\|c_{t+1}-\nabla F(w_{t})\| \leq \frac{\sqrt{2(\sigma^2+G^2)}}{\sqrt{(1-\beta_2)T}}+ \frac{2\sqrt{2}L \hat{G}\eta}{1-\beta_2}  + \frac{\sqrt{2}|\beta_1-\beta_2|}{\sqrt{1-\beta_2}}\sigma+\frac{1-\beta_1}{\sqrt{1-\beta_2}}\sigma,
	\end{align}
	where $\hat{G}=\max(G,\sqrt{d})$.
\end{lemma}

\begin{proof}
	At the lines 6 and 12 in Algorithm \ref{alg:2}, it has $c_{t}= \beta_{1}m_{t-1}+ (1-\beta_{1})g_t$ and $m_{t}= \beta_{2}m_{t-1}+ (1-\beta_{2})g_t$ for $t\geq1$.
	Since $c_{t+1}= \beta_{1}m_{t}+ (1-\beta_{1})g_{t+1}=\beta_{1}m_{t}+ (1-\beta_{1})\nabla f(w_{t};\xi_{t+1})$, then we have
	\begin{align} \label{eq:v1}
		& \E\|c_{t+1}-\nabla F(w_{t})\|^2 \nonumber \\
		&= \mathbb{E}\|\beta_{1}m_{t}+ (1-\beta_{1})\nabla f(w_{t};\xi_{t+1})-\nabla F(w_{t})\|^2 \nonumber \\
		& =  \mathbb{E}\|\beta_{1}\beta_{2}m_{t-1} + \beta_{1}(1-\beta_{2})\nabla f(w_{t-1};\xi_{t})+ (1-\beta_{1})\nabla f(w_{t};\xi_{t+1})-\nabla F(w_{t}) \|^2 \nonumber \\
		& = \mathbb{E}\|\beta_{2}(c_t-(1-\beta_1)\nabla f(w_{t-1};\xi_{t})) + \beta_{1}(1-\beta_{2})\nabla f(w_{t-1};\xi_{t})+ (1-\beta_{1})\nabla f(w_{t};\xi_{t+1})-\nabla F(w_{t}) \|^2 \nonumber \\
		& = \mathbb{E} \|\beta_2 c_t + (\beta_1-\beta_2)\nabla f(w_{t-1};\xi_{t})+ (1-\beta_{1})\nabla f(w_{t};\xi_{t+1})-\nabla F(w_{t}) \|^2 \nonumber \\
		& = \mathbb{E} \|\beta_2 (c_t-\nabla F(w_{t-1}))  +  (\beta_1-\beta_2)(\nabla f(w_{t-1};\xi_{t}) -\nabla F(w_{t-1}) )\nonumber \\
		& \quad + (1-\beta_{1})(\nabla f(w_{t};\xi_{t+1})-\nabla F(w_{t})) +\beta_1(\nabla F(w_{t-1})-\nabla F(w_{t})) \|^2 \nonumber \\
		& \mathop{=}^{(i)} \mathbb{E} \|\beta_2 (c_t-\nabla F(w_{t-1})) +\beta_1(\nabla F(w_{t-1})-\nabla F(w_{t})) \|^2\nonumber \\
		&\quad + \E\| (\beta_1-\beta_2)(\nabla f(w_{t-1};\xi_{t}) -\nabla F(w_{t-1}) )+ (1-\beta_{1})(\nabla f(w_{t};\xi_{t+1})-\nabla F(w_{t}))\|^2 \nonumber \\
		& \mathop{\leq}^{(ii)} \beta_2^2(2-\beta_2)\mathbb{E} \|c_t-\nabla F(w_{t-1})\|^2 + \beta_1^2(1+\frac{1}{1-\beta_2})\mathbb{E}\|\nabla F(w_{t-1})-\nabla F(w_{t})\|^2  \nonumber \\
		& \quad + 2(\beta_1-\beta_2)^2\sigma^2 + 2(1-\beta_1)^2\sigma^2 \nonumber \\
		& \mathop{\leq}^{(iii)} \beta_2\mathbb{E} \|c_t-\nabla F(w_{t-1})\|^2 + \frac{2L^2}{1-\beta_2}\mathbb{E}\|w_{t-1}-w_{t}\|^2  + 2((\beta_1-\beta_2)^2+(1-\beta_1)^2)\sigma^2 \nonumber \\
		& \leq \beta_2\mathbb{E} \|c_t-\nabla F(w_{t-1})\|^2 + \frac{8L^2\eta^2 \hat{G}^2}{1-\beta_2}  + 2((\beta_1-\beta_2)^2+(1-\beta_1)^2)\sigma^2,
	\end{align}
	where the equality $(i)$ holds by $\mathbb{E} [\nabla f(w_{t};\xi_{t+1})]=\nabla F(w_t)$ and $\mathbb{E} [\nabla f(w_{t-1};\xi_{t})]=\nabla F(w_{t-1})$ , and
	the inequality $(ii)$ holds by Young's inequality and Assumption~\ref{ass:v}, and the inequality $(iii)$ is due to $0< \beta_2 < 1$ such that  $\beta_2^2(2-\beta_2) =(1 -(1-\beta_2))^2(1+1-\beta_2)=1-(1-\beta_2)-(1-\beta_2)^2+
	(1-\beta_2)^3\leq 1-(1-\beta_2)$ and $\beta_2^2(1+\frac{1}{1-\beta_2}) \leq \frac{2}{1-\beta_2}$, and the last inequality holds by the above Lemma~\ref{lem:2}.
	
	By expanding the recursion to the above inequality~(\ref{eq:v1}), we can obtain
	\begin{align}
		&  \E\|c_{t+1}-\nabla F(w_{t})\|^2 \nonumber \\
		& \leq \beta_2^t\mathbb{E} \|c_1-\nabla F(w_{0})\|^2 + (\frac{8L^2\eta^2 \hat{G}^2}{1-\beta_2}  + 2((\beta_1-\beta_2)^2+(1-\beta_1)^2)\sigma^2)\sum_{s=1}^t\beta_2^{t-s} \nonumber \\
		& \leq \beta_2^t(2\sigma^2+2G^2) + \frac{8L^2\eta^2 \hat{G}^2}{(1-\beta_2)^2}  + \frac{2(\beta_1-\beta_2)^2}{1-\beta_2}\sigma^2+\frac{2(1-\beta_1)^2}{1-\beta_2}\sigma^2.
	\end{align}
	Then we have
	\begin{align}
		& \frac{1}{T} \sum_{t=1}^T\E\|c_{t+1}-\nabla F(w_{t})\|^2 \nonumber \\
		& \leq \frac{1}{T}\sum_{t=1}^T \big(\beta_2^t(2\sigma^2+2G^2) + \frac{8L^2\eta^2 \hat{G}^2}{(1-\beta_2)^2}  + \frac{2(\beta_1-\beta_2)^2}{1-\beta_2}\sigma^2+\frac{(1-\beta_1)^2}{1-\beta_2}\sigma^2)\big) \nonumber \\
		& \leq \frac{2(\sigma^2+G^2)}{T(1-\beta_2)}+ \frac{8L^2\eta^2 \hat{G}^2}{(1-\beta_2)^2}  + \frac{2(\beta_1-\beta_2)^2}{1-\beta_2}\sigma^2+\frac{(1-\beta_1)^2}{1-\beta_2}\sigma^2,
	\end{align}
	where the last inequality holds by $\sum_{t=1}^T\beta_2^t \leq \frac{1}{1-\beta_2}$.
	
	By using Jensen inequality, we have
	\begin{align}
		\frac{1}{T} \sum_{t=1}^T\E\|c_{t+1}-\nabla F(w_{t})\| & = \sqrt{\big(\frac{1}{T} \sum_{t=1}^T\E\|c_{t+1}-\nabla F(w_{t})\|\big)^2} \nonumber \\
		& \leq \sqrt{\frac{1}{T} \sum_{t=1}^T\E\|c_{t+1}-\nabla F(w_{t})\|^2} \nonumber \\
		& \leq \sqrt{\frac{2(\sigma^2+G^2)}{T(1-\beta_2)}+ \frac{8L^2\eta^2 \hat{G}^2}{(1-\beta_2)^2}  + \frac{2(\beta_1-\beta_2)^2}{1-\beta_2}\sigma^2+\frac{(1-\beta_1)^2}{1-\beta_2}\sigma^2} \nonumber \\
		& \leq \frac{\sqrt{2(\sigma^2+G^2)}}{\sqrt{(1-\beta_2)T}}+ \frac{2\sqrt{2}L \hat{G}\eta}{1-\beta_2}  + \frac{\sqrt{2}|\beta_1-\beta_2|}{\sqrt{1-\beta_2}}\sigma+\frac{1-\beta_1}{\sqrt{1-\beta_2}}\sigma.
	\end{align}
	
\end{proof}

\begin{theorem} (Restatement of Theorem~\ref{th:3})
	Assume the sequence $\{w_t\}_{t=1}^T$ is generated
	from Algorithm~\ref{alg:2}. Under the Assumptions~\ref{ass:s2},~\ref{ass:g},~\ref{ass:v},~\ref{ass:f}, and let $0<\lambda \leq \frac{1}{2\eta\hat{G} T^{\alpha}}$, $\eta=O(\frac{1}{T^{3/4}})$, $\beta_1=1-O(\frac{1}{\sqrt{T}})$, $\beta_2=1-O(\frac{1}{\sqrt{T}})$, $|\beta_1-\beta_2|=O(\frac{1}{\sqrt{T}})$ and $0<\nu_0 \leq \nu$,
	and further set $\alpha=\frac{5}{4}$ and $\nu_0 \geq O(\frac{1}{\sqrt{d}})$, we can obtain
	\begin{align}
		\frac{1}{T}\sum_{t=1}^T\E\|\nabla F(w_{t})\|_1 \leq O(\frac{\sqrt{d}}{T^{1/4}}).
	\end{align}
\end{theorem}

\begin{proof}
	When $\min_{ j\in S_t} |(c_t)_j| \geq \nu$, since $w_{t} = w_{t-1} - \eta (\sign(c_t)+\lambda w_{t-1})$, we have
	\begin{align}
		F(w_{t}) & \leq F(w_{t-1}) + \langle \nabla F(w_{t-1}), w_{t}-w_{t-1}\rangle + \frac{L}{2}\|w_{t}-w_{t-1}\|^2 \nonumber \\
		& = F(w_{t-1}) + \langle \nabla F(w_{t-1}), - \eta (\sign(c_t)+\lambda w_{t-1})\rangle + \frac{L}{2}\|w_{t}-w_{t-1}\|^2 \nonumber \\
		& = F(w_{t-1}) - \langle \nabla F(w_{t-1}),  \eta (\sign(c_t)-\sign(\nabla F(w_{t-1})))\rangle -  \eta \langle \nabla F(w_{t-1}), \sign(\nabla F(w_{t-1}))\rangle\nonumber \\
		& \quad - \eta\lambda \langle \nabla F(w_{t-1}), w_{t-1}\rangle+ \frac{L}{2}\|w_{t}-w_{t-1}\|^2 \nonumber \\
		& \mathop{\leq}^{(i)}  F(w_{t-1}) + 2\sqrt{d}\eta \|c_t- \nabla F(w_{t-1})\| -  \eta\|\nabla F(w_{t-1})\|_1 \nonumber \\
		& \quad - \eta\lambda \langle \nabla F(w_{t-1}), w_{t-1}\rangle+ \frac{L}{2}\|w_{t}-w_{t-1}\|^2 \nonumber \\
		& \leq  F(w_{t-1}) + 2\sqrt{d}\eta \|c_t- \nabla F(w_{t-1})\| -  \eta\|\nabla F(w_{t-1})\|_1 \nonumber \\
		& \quad + \eta\lambda \|\nabla F(w_{t-1})\|_1\|w_{t-1}\|_{\infty} +2L\eta^2 \hat{G}^2 \nonumber \\
		& \mathop{\leq}^{(ii)}  F(w_{t-1}) + 2\sqrt{d}\eta \|c_t- \nabla F(w_{t-1})\| -  \eta\|\nabla F(w_{t-1})\|_1 \nonumber \\
		& \quad + \eta\lambda \|\nabla F(w_{t-1})\|_1t\eta\hat{G} +2L\eta^2\hat{G}^2 \nonumber \\
		& \leq  F(w_{t-1}) + 2\sqrt{d}\eta \|c_t- \nabla F(w_{t-1})\| - \frac{\eta}{2}\|\nabla F(w_{t-1})\|_1  +2L\eta^2 \hat{G}^2,
	\end{align}
	where the above inequality $(i)$ holds by the Lemma~\ref{lem:s1} with $\alpha=1$, and the above inequality $(ii)$ holds by the Lemma~\ref{lem:2}, and the last inequality is due to $\lambda \leq \frac{1}{2\eta\hat{G} T^{\alpha}}$ with $\alpha>1$.
	Then we have
	\begin{align} \label{eq:r1}
		\|\nabla F(w_{t-1})\|_1 \leq \frac{2(F(w_{t-1})-F(w_{t}))}{\eta} + 4\sqrt{d}\|c_t- \nabla F(w_{t-1})\|+4L\eta \hat{G}^2.
	\end{align}
	
	When $\min_{ j\in S_t} |(c_t)_j| < \nu$, since $w_{t} = w_{t-1} - \eta (c_t+\lambda w_{t-1})$, we have
	\begin{align} \label{eq:ps}
		w_{t} & = w_{t-1} - \eta (c_t+\lambda w_{t-1}) \nonumber \\
		& = (1-\eta\lambda)w_{t-1} - \eta c_t \nonumber \\
		& = \mathop{\arg\min}_{w\in \R^d} \big\{ \langle c_t, w-(1-\eta\lambda)w_{t-1}\rangle +
		\frac{1}{2\eta}\|w-(1-\eta\lambda)w_{t-1}\|^2\big\}.
	\end{align}
	By using the optimality condition of the above subproblem~(\ref{eq:ps}), by using Lemma~\ref{lem:s2},
	we have
	\begin{align} \label{eq:F2}
		\langle c_t +\frac{1}{\eta}\big(w_{t}-(1-\lambda\eta)w_{t-1}\big),w-w_t\rangle \geq 0, \quad \forall w\in \mathbb{R}^d.
	\end{align}
	Putting $w=w_{t-1}$ into the above inequality~(\ref{eq:F2}), we have
	\begin{align} \label{eq:F3}
		\langle c_t +\frac{1}{\eta}\big(w_{t}-(1-\lambda\eta)w_{t-1}\big),w_{t-1}-w_{t}\rangle \geq 0.
	\end{align}
	Thus, we can obtain
	\begin{align} \label{eq:F4}
		\langle c_t,w_{t-1}-w_{t}\rangle &\geq \frac{1}{\eta}\langle w_{t}-w_{t-1},w_{t}-w_{t-1}\rangle + \lambda\langle w_{t-1},w_t-w_{t-1} \rangle \nonumber \\
		& \geq \frac{1}{\eta}\|w_{t}-w_{t-1}\|^2 + \lambda\langle w_{t-1},w_t-w_{t-1} \rangle.
	\end{align}
	Since $\min_{ j\in S_t} |(c_t)_j| < \nu$ with $S_t=\{j|(c_t)_j\neq 0, j=1,\cdots,d\}$, there exist a positive number $0<\nu_0\leq \nu$, we have $\nu_0\|c_t\|_1\leq \|c_t\|^2$. Then we have $-\|c_t\|^2 \leq -\nu_0 \|c_t\|_1$.
	
	When $\min_{ j\in S_t} |(c_t)_j| < \nu$, since $w_{t} = w_{t-1} - \eta (c_t+\lambda w_{t-1})$, we have
	\begin{align}
		F(w_{t}) & \leq F(w_{t-1}) + \langle \nabla F(w_{t-1}), w_{t}-w_{t-1}\rangle + \frac{L}{2}\|w_{t}-w_{t-1}\|^2 \nonumber \\
		& = F(w_{t-1}) + \langle \nabla F(w_{t-1})-c_t, w_{t}-w_{t-1}\rangle + \langle c_t, w_{t}-w_{t-1}\rangle + \frac{L}{2}\|w_{t}-w_{t-1}\|^2 \nonumber \\
		& \leq F(w_{t-1}) + \frac{\eta}{2}\|\nabla F(w_{t-1})-c_t\|^2 + \frac{1}{2\eta}\|w_{t}-w_{t-1}\|^2 + \langle c_t, w_{t}-w_{t-1}\rangle + \frac{L}{2}\|w_{t}-w_{t-1}\|^2 \nonumber \\
		& \mathop{\leq}^{(i)}  F(w_{t-1}) + \frac{\eta}{2}\|\nabla F(w_{t-1})-c_t\|^2 + \frac{1}{2\eta}\|w_{t}-w_{t-1}\|^2  \nonumber \\
		& \quad -\frac{1}{\eta}\|w_{t}-w_{t-1}\|^2 - \lambda\langle w_{t-1},w_t-w_{t-1} \rangle + \frac{L}{2}\|w_{t}-w_{t-1}\|^2 \nonumber \\
		& \leq  F(w_{t-1}) + \frac{\eta}{2}\|\nabla F(w_{t-1})-c_t\|^2 + \frac{1}{2\eta}\|w_{t}-w_{t-1}\|^2  \nonumber \\
		& \quad -\frac{1}{\eta}\|w_{t}-w_{t-1}\|^2 + \lambda\|w_{t-1}\|\|w_t-w_{t-1}\| + \frac{L}{2}\|w_{t}-w_{t-1}\|^2\nonumber \\
		& \leq  F(w_{t-1}) + \frac{\eta}{2}\|\nabla F(w_{t-1})-c_t\|^2 - \frac{1}{4\eta}\|w_{t}-w_{t-1}\|^2  + \lambda t\eta\hat{G}2\eta\hat{G} \nonumber \\
		& \mathop{\leq}^{(ii)}  F(w_{t-1}) + \frac{\eta}{2}\|\nabla F(w_{t-1})-c_t\|^2 - \frac{\eta}{8}\|c_{t}\|^2 + \frac{\lambda^2\eta}{4}\|w_{t-1}\|^2  +\frac{\eta \hat{G}}{T^{\alpha-1}} \nonumber \\
		& \mathop{\leq}^{(iii)}  F(w_{t-1}) + \frac{\eta}{2}\|\nabla F(w_{t-1})-c_t\|^2 - \frac{\eta\nu_0}{8}\|c_{t}\|_1 + \frac{\lambda^2\eta}{4}\|w_{t-1}\|^2  +\frac{\eta \hat{G}}{T^{\alpha-1}} \nonumber \\
		& \leq  F(w_{t-1}) + \frac{\eta}{2}\|\nabla F(w_{t-1})-c_t\|^2 - \frac{\eta\nu_0}{8}\|\nabla F(w_{t-1})\|_1 + \frac{\eta\nu_0}{8}\|c_t-\nabla F(w_{t-1})\|_1 \nonumber \\
		& \quad + \frac{\lambda^2\eta}{4}\|w_{t-1}\|^2  +\frac{\eta \hat{G}}{T^{\alpha-1}} \nonumber \\
		& \leq  F(w_{t-1}) + \frac{\eta}{2}\|\nabla F(w_{t-1})-c_t\|^2 - \frac{\eta\nu_0}{8}\|\nabla F(w_{t-1})\|_1 + \frac{\eta\nu_0\sqrt{d}}{8}\|c_t-\nabla F(w_{t-1})\| \nonumber \\
		& \quad + \frac{\lambda^2\eta}{4}t^2\eta^2\hat{G}^2  +\frac{\eta \hat{G}}{T^{\alpha-1}} \nonumber \\
		& \leq  F(w_{t-1}) + \frac{\eta}{2}\|\nabla F(w_{t-1})-c_t\|^2 - \frac{\eta\nu_0}{8}\|\nabla F(w_{t-1})\|_1 + \frac{\eta\nu_0\sqrt{d}}{8}\|c_t-\nabla F(w_{t-1})\| \nonumber \\
		& \quad + \frac{\eta}{16T^{2\alpha-2}} +\frac{\eta \hat{G}}{T^{\alpha-1}} ,
	\end{align}
	where the inequality $(i)$ holds by the above inequality~(\ref{eq:F4}), and the inequality $(ii)$ holds by $-\|w_{t}-w_{t-1}\|^2 = - \eta^2\|c_t + \lambda w_{t-1}\|^2 \leq - \frac{\eta^2}{2}\|c_t\|^2 + \eta^2\lambda^2\|w_{t-1}\|^2$, and the inequality $(iii)$ holds by $-\|c_t\|^2 \leq -\nu_0 \|c_t\|_1$ with $0<\nu_0\leq \nu$, and the last inequality is due to $\lambda \leq \frac{1}{2\eta\hat{G} T^{\alpha}}$ with $\alpha>1$.
	
	Then we have
	\begin{align} \label{eq:r2}
		\|\nabla F(w_{t-1})\|_1 & \leq \frac{8(F(w_{t-1})-F(w_{t}))}{\eta\nu_0} + \frac{4}{\nu_0}\|\nabla F(w_{t-1})-c_t\|^2 + \sqrt{d}\|c_t-\nabla F(w_{t-1})\| \nonumber \\
		& \quad + \frac{1}{2\nu_0T^{2\alpha-2}} +\frac{8\hat{G}}{\nu_0T^{\alpha-1}}
	\end{align}
	
	By combining the above inequalities~(\ref{eq:r1}) with~(\ref{eq:r2}), we have
	\begin{align}
		\E\|\nabla F(w_{t-1})\|_1 & \leq \frac{8(F(w_{t-1})-F(w_{t}))}{\eta\nu_0} + \frac{4}{\nu_0}\E\|\nabla F(w_{t-1})-c_t\|^2 + \sqrt{d}\E\|c_t-\nabla F(w_{t-1})\| \nonumber \\
		&\quad + \frac{1}{2\nu_0T^{2\alpha-2}} +\frac{8\hat{G}}{\nu_0T^{\alpha-1}} \nonumber \\
		& \quad + \frac{2(F(w_{t-1})-F(w_{t}))}{\eta} + 4\sqrt{d}\E\|c_t- \nabla F(w_{t-1})\|+4L\eta \hat{G}^2 \nonumber \\
		&= \frac{8(F(w_{t-1})-F(w_{t}))}{\eta\nu_0} + \frac{2(F(w_{t-1})-F(w_{t}))}{\eta}+ \frac{4}{\nu_0}\E\|\nabla F(w_{t-1})-c_t\|^2 \nonumber \\
		&\quad + 5\sqrt{d}\E\|c_t-\nabla F(w_{t-1})\|+ \frac{1}{2\nu_0T^{2\alpha-2}} +\frac{8\hat{G}}{\nu_0T^{\alpha-1}} +4L\eta \hat{G}^2.
	\end{align}
	Then we have
	\begin{align}
		& \frac{1}{T}\sum_{t=1}^T\E\|\nabla F(w_{t})\|_1 \nonumber \\
		&\leq \frac{1}{T}\sum_{t=1}^T\Big(\frac{8(F(w_{t-1})-F(w_{t}))}{\eta\nu_0} + \frac{2(F(w_{t-1})-F(w_{t}))}{\eta}\Big)+ \frac{4}{\nu_0}\frac{1}{T}\sum_{t=1}^T\E\|\nabla F(w_{t})-c_{t+1}\|^2 \nonumber \\
		& \quad 5\sqrt{d}\frac{1}{T}\sum_{t=1}^T\E\|c_{t+1}-\nabla F(w_{t})\|+\frac{1}{2\nu_0T^{2\alpha-2}} +\frac{8\hat{G}}{\nu_0T^{\alpha-1}} +4L\eta \hat{G}^2 \nonumber \\
		& \leq \frac{2(4+\nu_0)(F(w_{1})-F^*)}{T\eta\nu_0} + \frac{4}{\nu_0}\big(  \frac{2(\sigma^2+G^2)}{T(1-\beta_2)}+ \frac{8L^2\eta^2 \hat{G}^2}{(1-\beta_2)^2}  + \frac{2(\beta_1-\beta_2)^2}{1-\beta_2}\sigma^2+\frac{(1-\beta_1)^2}{1-\beta_2}\sigma^2\big) \nonumber \\
		& \quad + 5\sqrt{d}\big( \frac{\sqrt{2(\sigma^2+G^2)}}{\sqrt{(1-\beta_2)T}}+ \frac{2\sqrt{2}L \hat{G}\eta}{1-\beta_2}  + \frac{\sqrt{2}|\beta_1-\beta_2|}{\sqrt{1-\beta_2}}\sigma+\frac{1-\beta_1}{\sqrt{1-\beta_2}}\sigma\big) \nonumber \\
		& \quad + \frac{1}{2\nu_0T^{2\alpha-2}} +\frac{8\hat{G}}{\nu_0T^{\alpha-1}} +4L\eta \hat{G}^2,
	\end{align}
	where the last inequality holds by Assumption~\ref{ass:f}.
	
	Let $\eta=O(\frac{1}{T^{3/4}})$, $\beta_1=1-O(\frac{1}{\sqrt{T}})$, $\beta_2=1-O(\frac{1}{\sqrt{T}})$,  $|\beta_1-\beta_2|=O(\frac{1}{\sqrt{T}})$
	and $\alpha=\frac{5}{4}$, and let $\nu_0 \geq O(\frac{1}{\sqrt{d}})$, we can obtain
	\begin{align}
		\frac{1}{T}\sum_{t=1}^T\E\|\nabla F(w_{t})\|_1 \leq O(\frac{\sqrt{d}}{T^{1/4}}).
	\end{align}

\end{proof}

\end{document}